%% file: du_284.tex
\documentclass[accepted]{uai2021} 

\usepackage[american]{babel}

\usepackage{natbib} 
    \renewcommand{\bibsection}{\subsubsection*{References}}
\usepackage{mathtools} 
\usepackage{booktabs} 
\usepackage{tikz} 
\usepackage{amsmath}
\usepackage{amsthm}
\usepackage{amssymb}
\usepackage{algorithm}
\usepackage{subfig}
\usepackage{color}
\usepackage{graphicx}
\usepackage{grffile}
\usepackage{natbib}
\usepackage[pdf]{pstricks}
\usepackage{wrapfig,epsfig}
\usepackage{psfrag}
\usepackage{epstopdf}
\usepackage{url}
\usepackage{color}
\usepackage{epstopdf}
\usepackage{algpseudocode}
\usepackage{scrextend}
\usepackage[T1]{fontenc}
\usepackage{bbm}
\usepackage{comment}
\usepackage{dsfont}
\usepackage{bbold}
\usepackage{amsfonts}
\usepackage[nottoc,numbib]{tocbibind}
\settocbibname{References}
\usepackage{amsthm}
\usepackage{thmtools}
\usepackage{thm-restate}
\usepackage{xspace}
\usepackage{bm}
\usepackage{enumerate}
\usepackage{enumitem}

\usepackage{tikz}
\usepackage{hyperref}  
\usetikzlibrary{arrows}
\graphicspath{{./figs/}}

\newtheorem{theorem}{Theorem}[section]
\newtheorem{lemma}[theorem]{Lemma}

\newtheorem{definition}[theorem]{Definition}

\newtheorem{corollary}[theorem]{Corollary}

\newtheorem{assumption}[theorem]{Assumption}

\newcommand{\wh}{\widehat}
\newcommand{\wt}{\widetilde}

\newcommand{\eps}{\epsilon}

\newcommand{\R}{\mathbb{R}}
\renewcommand{\P}{\mathbb{P}}

\renewcommand{\d}{\mathrm{d}}

\newcommand{\norm}[1]{\left\lVert#1\right\rVert}
\renewcommand{\varepsilon}{\epsilon}
\renewcommand{\tilde}{\wt}
\renewcommand{\hat}{\wh}
\renewcommand{\eps}{\epsilon}

\renewcommand{\d}{\mathrm{d}}

\DeclareMathOperator*{\E}{{\mathbb{E}}}

\makeatletter
\DeclareRobustCommand\onedot{\futurelet\@let@token\@onedot}
\def\@onedot{\ifx\@let@token.\else.\null\fi\xspace}

\def\ie{i.e\onedot}

\makeatother

\title{When is Particle Filtering Efficient for Planning in\\Partially Observed Linear Dynamical Systems?}

\author[1]{Simon S. Du}
\author[2]{Wei Hu}
\author[2]{Zhiyuan Li}
\author[1]{\href{mailto:Ruoqi Shen <shenr3@cs.princeton.edu>?Subject=Your UAI 2021 paper}{Ruoqi Shen}{}}
\author[2]{Zhao Song}
\author[3]{Jiajun Wu}

\affil[1]{%
    University of Washington
}
\affil[2]{%
    Princeton University
}
\affil[3]{Stanford University}

\usepackage{xr}
\makeatletter
\newcommand*{\addFileDependency}[1]{
  \typeout{(#1)}
  \@addtofilelist{#1}
  \IfFileExists{#1}{}{\typeout{No file #1.}}
}
\makeatother

\newcommand*{\myexternaldocument}[1]{
    \externaldocument{#1}
    \addFileDependency{#1.tex}
    \addFileDependency{#1.aux}
}

\myexternaldocument{du_284-supp}

\begin{document}
\maketitle

\begin{abstract}
Particle filtering is a popular method for inferring latent states in stochastic dynamical systems, whose theoretical properties have been well studied in machine learning and statistics communities. In many control problems, e.g., partially observed linear dynamical systems (POLDS), oftentimes the inferred latent state is further used for planning at each step. 
This paper initiates a rigorous study on the efficiency of particle filtering for sequential planning, and gives the first particle complexity bounds. Though errors in past actions may affect the future, we are able to bound the number of particles needed so that the long-run reward of the policy based on particle filtering is close to that based on exact inference. In particular, we show that, in stable systems, polynomially many particles suffice.
Key in our proof is a coupling of the ideal sequence based on the exact planning and the sequence generated by approximate planning based on particle filtering. We believe this technique can be useful in other sequential decision-making problems.
\end{abstract}

\section{Introduction}
Many real-world applications require planning on a partially observed stochastic dynamic system~\citep{kaelbling1998planning}.
The planning policy often operates on the underlying latent states instead of directly on raw observations.
Take robot navigation as an example. The raw observations are high-dimensional RGB-D videos, and it is often preferred to instead plan upon the underlying latent state, such as the location of the robot. 

A core challenge is to infer these latent states from observations.
For simple stochastic systems such as hidden Markov models (HMM) corrupted by a Gaussian noise, there are analytical solutions for inference, \ie,  Kalman filtering~\citep{kalman1960new}.
However, exact inference is often computationally infeasible in many stochastic systems with complex probabilistic models.
A typical example is inferring the latent state of a partially observable linear dynamical system (POLDS), especially in recent models that  parametrize the transition and emission probability distributions with deep neural networks~\citep{hausknecht2015deep}.
Here, while exact computation of the transition kernel and the stochastic emission kernel is efficient (the same computation complexity as using deep neural networks for prediction), exact computation of the posterior distribution over latent states is infeasible.

Particle filtering or Sequential Monte Carlo is a generic approach to \emph{approximately} infer the underlying latent states in stochastic dynamical systems (cf. Algorithm~\ref{alg:pf}).
Instead of computing the posterior distribution exactly,  this approach simulates a set of particles according to the transition kernel. Then, a weighted average of the particles is used to approximate the posterior distribution, where the weight of each particle is given by its likelihood. 
Particle filtering is computationally efficient because it only needs to compute the transition kernel and the stochastic emission kernel, but does not require computing the posterior distribution.

%
Particle filtering as approximate inference of latent states can be naturally integrated with \emph{belief space planning}~\citep{platt2010belief}. Recently, researchers have also proposed various approximations to make the steps within particle filtering differentiable, so that the inference networks can be trained end-to-end with policy networks~\citep{karkus2017qmdp,karkus2018integrating,karkus2018particle,jonschkowski2018differentiable,wang2019dual}. 
In terms of applications, however, these works mostly focus on visual navigation, where the planning horizon is short (with instant feedback), and the reward function varies smoothly and continuously with respect to actions such as moving forward. Applications in dynamic systems without these properties are rare. 
This gives rise to a theoretical question:\begin{center}
\textbf{
	Is particle filtering \emph{provably efficient} for sequential planning on stochastic systems?	
	}
\end{center}
While the theory of particle filtering for inference is well-studied in statistical machine learning, the theory of particle filtering for planning is rather unexplored.
The approximation error in inference can lead to selection of different actions and further affect the outcome such as cumulative rewards in the future. 
Therefore, we not only need to study the approximation in the inference, but also how the error  affects the future planning.

In this paper, we initiate the rigorous quantitative study to characterize the efficiency of particle filtering in terms of the properties of stochastic system.
We study the fundamental hidden Markov model, where the dynamics of transition and emission are linear, but the noise in transition and emission can be arbitrary probabilistic distributions.
We focus on the planning problem in which we assume noise distributions are known.
Unless these noise distributions are within specific classes such as Gaussian, exact inference for latent states is computationally infeasible, and approximate inference such as particle filtering is needed.
Our analysis not only applies to popular linear, time-invariant (LTI) systems, but also time-varying ones. It can also potentially be extended for nonlinear dynamic systems with recent development in Koopman analysis~\citep{brunton2016koopman}, which will enable additional applications in robotic control~\citep{bruder2019modeling}. 

\paragraph{Our Contribution}  Our main contribution is an upper bound on the number of particles needed for particle filtering--based planning to be close to the planning based on the exact inference.
To our knowledge, this is the first non-asymptotic theoretical result on particle filtering for sequential planning.
The bound depends on some control-theoretic quantities that describe the POLDS, the planning horizon,  the Lipschitzness of the reward function, and the inverse of the likelihood of observations and the target sub-optimality. 
We also complement the upper bound with a lower bound showing the dependency is necessary.

\paragraph{Main Challenge and Analysis Overview}
The main challenge in the analysis is studying the distribution of the particles. When there is no sequential planning, the particles are generated independently, so their distributions can be easily studied. However, when doing sequential planning, i.e., when the states of the particles depend on the past actions, the particles are not independent anymore. The actions taken is based on the particle approximation of the past states, so the particles are correlated with each other. To avoid analyzing the complicated joint distribution of the particles directly, we show that it is enough to analyze the particle approximation of the noise in each round separately. The  simulated noise in each particle is independent and can be easily analyzed.   

To study the performance of particle filtering--based sequential planning, we need to compare the approximate process generated by particle filtering with an ideal process generated by exact inference. To make sure that the two processes can be fairly compared, we couple the approximate sequence with the ideal sequence using the same noise. It can be hard to compare those two processes because after taking different actions, the two processes are not estimating the same state anymore. In the following time steps, the two processes will take actions based on estimations of different states. Then, even when there is no estimation error present, two processes can still grow further apart. In this paper, we show that although the error can accumulate and be amplified through actions in sequential planning, we can still upper bound the number of particles needed so that the long-run rewards of the two processes remain close. We believe our framework can be the starting point of future study on particle filtering for sequential planning and can be useful for studying other methods.

\paragraph{Organization} This paper is organized as follows.
In Section~\ref{sec:rel}, we discuss related works.
In Section~\ref{sec:pre}, we introduce necessary notations and formally state the problem.
In Section~\ref{sec:result}, we present our main result, an upper bound on the particle complexity of particle filtering -- based sequential planning.
In Section~\ref{sec:lowerbound}, we complement the upper bound with a lower bound.
In Section~\ref{sec:dis}, we conclude and list concrete open problems. 
In Appendix, we show the proofs that are omitted in the main paper and present some simulation studies.

\subsection{Related Work}
\label{sec:rel}
Here we discuss related theoretical work.

For inference, the quality of particle filtering is often measured by the distance between the posterior from the exact inference and that from the approximate inference, in metrics such as $L_2$ distance and Kullback-Leibler divergence.
Many works have analyzed the number of particles needed to make the distance small, conditioned on properties of the dynamic system~\citep{whiteley2016role,huggins2019sequential,crisan2002survey,marion2018finite,chopin2004central,oreshkin2011analysis}.
However, to our knowledege, no prior work analyzed the quality of particle filtering--based planning.
In this setting, the distance between the posterior from exact inference and approximate inference is not sufficient to measure for quality of particle filtering, as the approximation error in inference can lead to the selection of a different action and in turn affect the total reward.

Controlling a known dynamical system is a classical problem~\citep{bertsekas1995dynamic}.
For the POLDS setting considered in this paper, the controller often needs to first infer the latent state and plan on top of it.
When the noise distribution is Gaussian and the reward is quadratic, the problem becomes Linear-Qudratic-Gaussian (LQG) control.
For LQG, one can first use analytical formulas for inference~\citep{kalman1960new}, and then, by the separation principle, directly apply an linear controller on the inferred latent state.
Unfortunately, for most noise distributions, exact inference is in general computationally intractable and we must resort to approximate inference techniques such as particle filtering.
Recently, researchers also try to leverage online learning techniques to design provable algorithms for control with known and unknown linear dynamical systems~\citep{agarwal2019online,agarwal2019logarithmic,li2019online,cohen2018online,even2009online,goel2020power,yu2009markov,abbasi2014tracking,neu2017fast,foster2020logarithmic,dean2019robust,tsiamis2020sample,hazan2017learning,simchowitz2020improper,simchowitz2020making}, some of which can be even applied to adversarial noise.
Our work differs from this line of research as particle filtering--based planning is a fundamentally different approach.

\section{Preliminaries}
\label{sec:pre}

\paragraph{Notations}
For any positive integer $n$, we use $[n]$ to denote the set of integers $\{1,2,\cdots,n\}$. For vector $x$, we use $\| x \|$ to denote its $\ell_2$ norm. For matrix $A$, we use $\| A \|_{op}$ denote the operator norm of $A$.  For time $t$, we use $x_{0:t}$ to denote the sequence $x_0,...,x_t$. 
For event $E$, we use $\Pr[E]$ to denote the probability that the event $E$ happens. For random variable $X, Y$ and their realizations $x, y$, we use $\P_X[x]$ to denote the value of the probability density function of $X$ at $x$ and $\P_{X}[x|y]$ to denote the value of the density function of $X$ conditional on $Y = y$ at $x$. We write $\P_X[x]$ and $\P_{X}[x|y]$ as $\P[x]$ and $\P[x|y]$ when there is no ambiguity. 
For any function $f$, we use $\tilde{O}(f)$ to denote the class $O(f)\cdot \log^{O(1)}(f)$.

\paragraph{Problem Setup}
\label{subsec: probsetup}
\begin{algorithm}[tb]
	\caption{Particle Filtering for Sequential Planning}
\label{alg:pf}
\begin{algorithmic}
\State {\bfseries Input:} starting state $x_0$,  number of particles $N$, number of time steps $T$.

\For{$i = 1 \to N$}
\State Initialize particle weight $w_0^{(i)} = 1$. 
\State Initialize particle state $x_0^{(i)} = x_0$.
\EndFor

\For{$t = 0 \to T-1$}
 \State Estimate latent state $\hat{y}_t \gets \frac{\sum_{i=1}^N  w_t^{(i)} x_t^{(i)}}{\sum_{i=1}^N w_t^{(i)} }$.
 \State Take action $\hat{u}_t = g(\hat{y}_t)$ and observe $o_{t+1}$.
 \For{$i = 1 \to N$}
 \State Generate random noise $\xi_t^{(i)} \sim \mu_t(\cdot )$.
 \State Update particle state $x_{t+1}^{(i)} = A_tx_t^{(i)} + B_t\hat{u}_t + \xi_t^{(i)}$.
 \State Update particle weight $w_{t+1}^{(i)} \gets w_t^{(i)} \cdot \eta_{t+1}(o_{t+1} - C_{t+1}x_{t+1}^{(i)} ) $.
 \EndFor
\EndFor
\end{algorithmic}
\end{algorithm}

We study the setting of planning in POLDS. At each time step $t=1,...,T$, the environment is in some latent state $x_t\in \mathcal{X}\subseteq \R^d$. The agent receives a partial observation $o_t\in \mathcal{O} \subseteq \R^m$ of the latent state $x_t$ and takes action $\hat{u}_t\in \mathcal{U}\subseteq{\R^k}$ based on the observations in the current time step and previous time steps, $o_{0:t}$. The action causes the environment to change to the new state $x_{t+1}$ based on a known transition kernel. Finally, in time step $T$, we receive a reward $R$, which is a function of the past states and actions.

To put it formally, in our setting, at $t=0$, we start from a known state $x_0$, which can be observed exactly. For $t=0,...,T-1$, we have the state updated as 
\begin{align}\label{def:x_t+1}
x_{t+1} = A_t \cdot x_{t}+B_t \cdot \hat{u}_{t}+\xi_{t}.
\end{align}
$\hat{u}_{t}$ is the action we take at time $t$. $A_t \in \R^{d \times d}$ and $B_t \in \R^{d \times k}$ are transition matrices on the state $x_t$ and the action  $\wh{u}_t $, respectively, at time $t$. We assume that the matrices $A_t$ and $B_t$ are known. $\xi_{t} \in \R^{d}$ is some transition noise following a known distribution  $\xi_{t}\sim\mu_t(\cdot)$. The action $\hat{u}_t$ is taken based on the observations $o_{1:t}$ of the current state and the past states. 

At state $x_t$, the observation $o_t$ is given by 
\begin{align}\label{def:o_t+1}
o_{t} = C_t \cdot x_{t} + \zeta_{t}.
\end{align}
where $C_t\in \R^{m \times d}$ is a known transition matrix and $\zeta_{t}$ is some  noise following a known distribution $\eta_t(\cdot)$. 

In our setting, we are given a policy $g:\mathcal{X} \rightarrow\mathcal{U}$, which is a function on the state space. However, since we only have access to a partial observation of the latent state $x_t$, we can only infer the state $x_t$ based on the observations $o_1,...,o_{t}$. We use particle filtering to do the latent state inference, which is listed in Algorithm~\ref{alg:pf}.

Particle filtering estimates the latent state by simulating a group of particles using the known transition kernel. Those particles update their states using the same actions as we take in the real process. Then, the state estimation is given by a weighted average of the simulated states of the particles. The weight of each particle is proportional to the likelihood of the states of that particle given the observations. 

Formally, we simulate $N$ particles, $x_{0:T}^{(1)},...,x_{0:T}^{(N)} \in \mathcal{X}^T$. All $N$ particles start from the same starting state $x_0$. In time step $t$, the particles are updated according to
\begin{align}\label{def:x_t+1_i}
x_{t+1}^{(i)} = A_t \cdot x_{t}^{(i)} + B_t \cdot \hat{u}_{t} + \xi_{t}^{(i)}.
\end{align}
The action $\hat{u}_{t}$ is the same as the action taken in step $t$ of the real process $x_{0:T}$. $\xi_{t}^{(i)}$ are sampled independently according to the noise distribution $\mu_t(\cdot)$. 

Next, we show how, at time $t$, we use the simulated particles $x_{0:t}^{(1)},...,x_{0:t}^{(N)} $ to estimate the latent state $x_t$. The weight of particle $i$ at time $t>0$, $w_t^{(i)}$, is given by 
\begin{align*}
w_{t}^{(i)} & =\prod_{s=1}^{t} \P \left[ o_s ~|~ x_s^{(i)} \right] = \prod_{s=1}^{t} \eta_s\left( o_s -  C_s \cdot x_s^{(i)} \right).
\end{align*}
The weight $w_{t}^{(i)}$ of the $i$-th particle measures how likely the true latent states, $x_{0:t}$, are the states of the particle $i$, $x_{0:t}^{(i)}$,  given the observations $o_{1:t}$. We give higher weights to particles with more likely states. Then, the estimated state $\hat{y}_t$ is a weighted average of the states of the particles,
\begin{align*}
\hat{y}_t = \frac{\sum_{i=1}^N w_{t}^{(i)} x_t^{(i)}}{\sum_{i=1}^N w_{t}^{(i)}}.
\end{align*} 
We note that if an infinite amount of particles are simulated, the estimated state would be the posterior mean of the state given the observations. Given the estimated state $\hat{y}_t$, we take action $\hat{u}_t$ to be 
\begin{align*}
\hat{u}_t = g(\hat{y}_t).
\end{align*}
Note we study policies that only depend on the estimated hidden state, which follows the separation principle in stochastic control theory. This class of policies is optimal for certain settings~\citep{bertsekas1995dynamic}.


To study the efficiency of the particle filtering algorithm, we compare the approximate process, ${x}_{0:t}$, described above, with an ideal process, $x^*_{0:t}$, which are generated via exact inference. The ideal process starts from $x^*_0=x_0$.
For $t=0,...,T-1$, the ideal process is updated as 
\begin{align}\label{eq:x_t+1_*}
x_{t+1}^{*} = A_t \cdot x_{t}^{*} + B_t \cdot u_{t}^{*} + \xi_{t},
\end{align}
The action $u_{t}^{*}$ is taken based on exact inference, which will be defined formally later. Similarly, the observation $o^*_t$ for $t=1,...,T$, is generated according to
\begin{align}\label{eq:o_t+1_*}
o_{t}^{*} = C_t \cdot x_{t}^{*} + \zeta_{t},
\end{align}
The transition matrices $A_t$, $B_t$ and $C_t$ are the same as those used in generating the approximate process. To make sure that the two processes can be fairly compared, we let the transition noise $\xi_t$ and the observation noise $\zeta_t$ in the ideal process be the same as those in the approximate process.

Now, we show how the action $u^*_t$ is chosen in the ideal process.
 We assume that in the ideal process,
 we can compute the exact posterior mean of the state $x^*_t$ given the observations $o^*_{1:t}$. We estimate the state $x^*_t$ as
\begin{align*}
\tilde{y}_t = & \frac{\int_{x'_{1:t}\in \mathcal{X}^t}\prod_{s=1}^t \P\left[o^*_s~|~x'_s\right] x'_t\d \rho_t(x'_{1:t})}{\int_{x'_{1:t}\in \mathcal{X}^t}\prod_{s=1}^t \P\left[o^*_s~|~x'_s\right] \d \rho_t(x'_{1:t})} \\
=  &\frac{\int_{x'_{1:t}\in \mathcal{X}^t}\prod_{s=1}^t \eta_s\left( o^*_s - C_s \cdot x'_s\right) x'_t \d \rho_t(x'_{1:t})}{\int_{x'_{1:t}\in \mathcal{X}^t}\prod_{s=1}^t \eta_s\left( o^*_s - C_s \cdot x'_s\right) \d \rho_t(x'_{1:t})}.
\end{align*}
where $\rho_t$ is the distribution of $x'_{1,t}$ given actions $u^*_{0:t-1}$ and starting state $x_0$. Given the estimation $\tilde{y}_t$, the action $u^*_t$ is taken to be
$
	u^*_t = g(\tilde{y}_t).
$

In this paper, we study how accurately particle filtering can approximate the exact inference and how the error of particle filtering can affect the long-run reward where the reward function $r_T: \mathcal{X}^T \times \mathcal{U}^T\rightarrow \R $ maps some states $x_{1:T}$ and actions $u_{0:T-1}$ in the past $T$ time steps to a reward value in $\R$.
In particular, we study the number of particles needed so that the reward at time step $T$, $r_T({x}_{1:t}, \hat{u}_{0:T-1})$, of the approximate process is close to that of the ideal process, $r_T({x}^*_{1:t}, {u}^*_{0:T-1})$.

\section{Main Results}
\label{sec:result}
In this section, we present our main theoretical results.
First, we  introduce some necessary regularity conditions.
Next, we discuss our results for general non-linear policies.
Lastly, we focus on linear policies and give more refined results.

\subsection{Regularity Assumptions} 
To formally state our results, we first describe our assumptions.

\begin{assumption}
	\label{asmp:subgaussian}
	The  transition noise $\xi \in \R^d$ is sub-Gaussian with parameter $1/m$, i.e., 
	$$\E \left[e^{u^\top (\xi - \E{\xi})} \right]\leq e^{\|u\|^2/(2m)}, \text{ for any vector } u \in \R^d.$$
\end{assumption}
 Assumption~\ref{asmp:subgaussian} is standard regularity condition on transition noise. Without a regularity condition, the noise can arbitrarily large.

\begin{assumption}
	\label{asmp:reward_lip}
	The reward function $r_T$ is $L_r$-Lipschitz, i.e.,
	\begin{align*}
	& ~| r_T(x_{1:t},u_{0:t-1}) - r_T(x'_{1:t},u'_{0:t-1}) | \\
	&\leq ~  L_r \cdot \left( \sum_{t=1}^T\| x_t - x'_t \| +\sum_{t=0}^{T-1} \| u_t -  u'_t \| \right),
	\end{align*}
	for all  $x_{1:T}, x'_{1:T}\in \mathcal{X}^T$  and $u_{0:T-1}, u'_{0:T-1}\in \mathcal{U}^T$.
\end{assumption}
Assumption~\ref{asmp:reward_lip} is regularity condition imposed on the reward condition.
Note since we are considering planning based on approximate inference, we cannot hope to choose the same action as the one based on the exact inference.
Similarly, we cannot hope to have same sequence of hidden states based on particle filtering as that based on the exact inference.
Assumption~\ref{asmp:reward_lip} bounds how much small deviation on the action and hidden state sequence will affect the reward.

\subsection{Main Result for General Non-linear Policies}
Now we discuss our result on non-linear policies.
We need some assumptions about the dynamical system and the policy to characterize the stability of the system.
Such assumptions are necessary for control problems.
For the non-linear policy, we consider the following assumption.
\begin{assumption}
	\label{asmp:g_lip}
	The policy $g$ is $L_g$-Lipschitz. i.e., for all $x, x' \in \mathcal{X}$, $
	\| g(x) - g(x') \| \leq L_g \cdot \| x - x' \|,$
	for all $0 \leq t_1 <  t_2 <T$,  $\norm{\Pi_{s=t_1}^{t_2} A_s}_{op} \le C_{a} \rho_{a}^{t_2 - t_1},$   and for all $0 \leq t<T $, $\norm{B_t}_{op} \leq C_b$ for some $L_g,C_a,\rho_a,C_b > 0$.
\end{assumption}
Assumption~\ref{asmp:g_lip} imposes a Lipschitz condition on the policy $g$. Note for this assumption, the policy $g$ can be non-linear.
The condition $\norm{\Pi_{s=t_1}^{t_2} A_s}_{op} \le C_{a} \rho_{a}^{t_2-t_1}$ describes the growth rate of the dynamical system.
Such assumption is standard in the control literature.
The bound on $\norm{B_t}_{op}$ ensures a small deviation on the action will not alter the system by much.


Our main result is the following.
\begin{theorem}\label{thm:nonlinear}
Given any accuracy $\epsilon \in (0,1/2)$, failure probability $\delta>0$ and number of time steps $T \geq 1$. 
		Under Assumptions~\ref{asmp:subgaussian},~\ref{asmp:reward_lip} and~\ref{asmp:g_lip},  let \[\Sigma_a^{(T)} =1 +  C_a\sum_{s=0}^{T-2} \rho_a^s \text{  and  }  \Sigma_{ab}^{(T-1)} =\sum_{s=0}^{T-2} (C_a + C_bL_g)^{s}.\]  Let \[\Delta_T =  L_rL_g \Sigma_a^{(T)} \left( 1 + C_b\Sigma_a^{(T)}\right)\left( 1 +L_g C_b\Sigma_{ab}^{(T-1)}\right)   \] and  \[ p = \P_{O_{1:T}} \left[ o_{1:T} ~|~ \hat{u}_{0:t-1}, x_0 \right].\] 
For any $\delta > 0$, it is enough to use 
\[
	N =\tilde{ O}  ( T^2  \Delta_T^2  d m^{-1}  \epsilon^{-2} p^{-1} )\] particles so that with probability at least $1 - \delta$,
		\[
		|r_T(x_{1:T}, \hat{u}_{0:T-1}) - r_T(x^{*}_{1:T}, u^{*}_{0:T-1})|  \leq  \epsilon.
		\]
\end{theorem}
Theorem~\ref{thm:nonlinear} shows as long as the number of particles scales \emph{polynomially} with $p$, parameters in Assumptions~\ref{asmp:subgaussian} and \ref{asmp:reward_lip} and a quantity $\Delta_T$ defined by the parameters in Assumption~\ref{asmp:g_lip}, the reward collected by particle filtering--based planning is close to that of the ideal process. Here, $p$ is the likelihood of the observations $o_{1:T}$ conditional on the initial state and the actions. When the space of the observation is discrete, $p$ is the probability of seeing the observations. We note that in some cases, it is possible that $1/p$ grows exponentially as the number of time steps $T$ grows. However, we are able to show it is necessary for the number of particles to depend on $1/p$. The lower bound on the dependence on $1/p$ is stated in Section~\ref{sec:lowerbound}.
To our knowledge, this is the first non-asymptotic particle complexity analysis for planning problems.

To prove the theorem, we first show the number of particles needed so that the particle can approximate the latent state, especially the transition noise, accurately. Then, we show how the error in each time step can accumulate through the  planning process. We show the proof ideas in Section~\ref{sec:proof_idea} and defer the complete proof to Appendix. 

Our bound depends on $\Delta_T$ which in turn depends on two quantities $\Sigma_a^{(T)}$ and $\Sigma_{ab}^{(T-1)}$ which together describe the growth rate of the system, i.e., how stable the system is.
To better illustrate Theorem~\ref{thm:nonlinear}, we consider the benign scenario where the system is stable in the sense that $\rho_a \le 1$ and $C_a+C_bLg \le 1$.
Stable systems are widely studied in the control literature.
The following corollary shows if the system is stable, then the number of particles only needs to scale polynomially with all parameters.
\begin{corollary} \label{cor:nonlinear}
In the same setup as Theorem~\ref{thm:nonlinear}, suppose $\rho_a \le 1$ and $C_a+C_bL_g \le 1$.
Then it is enough to use \[N=\tilde{ O}  ( T^6 d m^{-1} L_r^2L_g^2(1+C_b^2T^2) \epsilon^{-2} p^{-1} )\] particles so that for any $\delta > 0$, with probability at least $1-\delta$, $	|r_T(x_{1:T}, \hat{u}_{0:T-1}) - r_T(x^{*}_{1:T}, u^{*}_{0:T-1})|  \leq  \epsilon.$
%
\end{corollary}

\subsection{Main Result for Linear Policies}
In this section, we focus on the linear policy, i.e., $g(x) = Gx$ for some matrix $G$.
Linear policy is a popular class and is widely studied in the control and online learning literature.
We make the following assumption about the policy and the system.
\begin{assumption}
	\label{asmp:g_linear}
	$\norm{G}_{op}\leq L_g$, for all $0\leq t_1 <  t_2 < T$,  $\norm{\Pi_{s=t_1}^{t_2}\left(A_s+B_s G\right)}_{op} \le C_{ab}{\rho}_{ab}^{t_2 - t_1}$, $\norm{\Pi_{s=t_1}^{t_2} A_s}_{op} \le C_{a} \rho_{a}^{t_2 - t_1}$, and for all $0\leq t < T$,  $\norm{B_t}_{op} \leq C_b$, and $\norm{B_tG}_{op} \leq C_{bg}$ for some $L_g, C_{ab},\rho_{ab},C_a,\rho_a, C_b, C_{bg} > 0$.
\end{assumption}
Assumption~\ref{asmp:g_linear} can be viewed as a fine-grained version of Assumption~\ref{asmp:reward_lip}.
Recall that Theorem~\ref{thm:nonlinear} depends on $\left(C_a+C_bL_g\right)$ which corresponds to the condition $\norm{\Pi_{s=t_1}^{t_2}\left(A_s+B_s G\right)}_{op} \le C_{ab}{\rho}_{ab}^{t_2 - t_1}$.
Note $\left(C_a+C_bL_g\right)^{t_2-t_1}$ is always an upper bound of $\norm{\Pi_{s=t_1}^{t_2}\left(A_s+B_s G\right)}_{op}$, so the condition $\norm{\Pi_{s=t_1}^{t_2}\left(A_s+B_s G\right)}_{op} \le C_{ab}{\rho}_{ab}^{t_2 - t_1}$ is a more refined characterization.
We remark that this condition is also a common one in the control literature.
Similarly, Theorem~\ref{thm:nonlinear} depends on $L_gC_b$, which corresponds to the condition $\norm{B_tG}_{op} \leq C_{bg}$ in Assumption~\ref{asmp:g_linear}.
$C_{bg}$ is a more refined characterization of $\norm{B_tG}_{op} $ than $L_gC_b$. 
Now we present our general theoretical result.

\begin{theorem}\label{thm:main_linear}
	 For any accuracy $\epsilon \in (0,1/2)$, failure probability $\delta>0$ and number of time steps $T \geq 1$. 
	 Under Assumption~\ref{asmp:subgaussian},~\ref{asmp:reward_lip} and~\ref{asmp:g_linear},  let $\Sigma_a^{(T)} =1 +  C_a\sum_{s=0}^{T-2} \rho_a^s$ and $\bar{\Sigma}_{ab}^{(T-1)} =1 + C_{ab}\sum_{s=0}^{T-3} \rho_{ab}^{s}$.  Let $\Delta_T =L_rL_g \Sigma_a^{(T)} \left( 1 + C_b\Sigma_a^{(T)}\right)\left( 1 +C_{bg}\bar{\Sigma}_{ab}^{(T-1)}\right) .$ 
For any $\delta > 0$, it is enough to use
\[
	N =\tilde{ O}  ( T^2  \Delta_T^2  d m^{-1}  \epsilon^{-2} p^{-1} )
\]
	particles so that with probability at least $1 - \delta$,
	\[
	|r_T(x_{1:T}, \hat{u}_{0:T-1}) - r_T(x^{*}_{1:T}, u^{*}_{0:T-1})|  \leq  \epsilon.
	\]
\end{theorem}
Similar to Theorem~\ref{thm:nonlinear}, Theorem~\ref{thm:main_linear} also guarantees that with a sufficiently large number of particles, the reward collected by particle filtering--based planning is close to that of the ideal process.
The main difference is that Theorem~\ref{thm:main_linear} depends on parameters in Assumption~\ref{asmp:g_linear}, which are finer-grained characterizations of the process.
This requires some new proof components that exploit the linearity of the policy.

Again, to better illustrate our result, we provide the following corollary for the stable system.


\begin{corollary} \label{cor:linear}
	In the same setup as Theorem~\ref{thm:nonlinear}, suppose $\rho_a \le 1$ and $\rho_{ab}\le 1$.
	Then it is enough to use \begin{align*}
	N=\tilde{O}  \big(& T^2 d m^{-1} L_r^2L_g^2 \epsilon^{-2} p^{-1} (1+C_a^2T^2)\\&(1+C_b^2+C_b^2C_a^2T^2)(1+C_{bg}^2+C_{bg}^2C_{ab}^2)\big)
	\end{align*}particles so that for any $\delta > 0$, with probability at least $1-\delta$, \[|r_T(x_{1:T}, \hat{u}_{0:T-1}) - r_T(x^{*}_{1:T}, u^{*}_{0:T-1})|  \leq  \epsilon.\]
\end{corollary}
The conditions $\rho_a \le 1, \rho_{ab} \le 1$ appeared in many studies on linear systems. 
Corollary~\ref{cor:linear} guarantees that for these systems, the number of particles only needs to scale polynomially with all parameters.

\subsection{Proof of Main Results} 
\label{sec:proof_idea}
In this section, we state the proof idea of our main results, Theorem~\ref{thm:nonlinear} and Theorem~\ref{thm:main_linear}. We defer a complete version of the proof to Appendix.

We first note that at time $t$, since we know the initial state $x_0$, the transition matrices $A_{0:t-1}$ and $B_{0:t-1}$ and the past actions $\hat{u}_0,...,\hat{u}_{t-1}$, estimating the state $x_t$ is equivalent to estimating $\xi_0,...,\xi_{t-1}$. 
\begin{restatable}{lemma}{decomposition}
	\label{lem:decomposition}
	For any $t\in [T]$, we can write the state $x_t$ as 
	\begin{align}\label{eq:decomposition}
	x_t = \sum_{s=0}^{t-1} \prod_{s'=s+1}^{t-1} A_{s'} \cdot \left(\xi_{s} + B_{s}\cdot \hat{u}_{s}\right) + \prod_{s=0}^{t-1} A_{s}\cdot x_0,
	\end{align}
	the state $x^*_t$ as 
	\begin{align}\label{eq:decomposition2}
	x^*_t = \sum_{s=0}^{t-1} \prod_{s'=s+1}^{t-1} A_{s'} \cdot \left(\xi_{s} + B_{s}\cdot {u}^*_{s}\right) + \prod_{s=0}^{t-1} A_{s}\cdot x_0,
	\end{align}
	and for particle $i\in[N]$,
	\begin{align}\label{eq:decomposition3}
	x_t^{(i)}  = \sum_{s=0}^{t-1} \prod_{s'=s+1}^{t-1} A_{s'} \cdot \left(\xi_{s}^{(i)}+ B_{s}\cdot \hat{u}_{s}\right) + \prod_{s=0}^{t-1} A_{s}\cdot x_0.
	\end{align}
\end{restatable}
\begin{proof}
	\eqref{eq:decomposition} follows directly from applying the definition of the process given in  \eqref{def:x_t+1} recursively. Similarly, \eqref{eq:decomposition2} and \eqref{eq:decomposition3} can be obtained by the definitions \eqref{eq:x_t+1_*} and \eqref{def:x_t+1_i}.
\end{proof}

Recall from Section~\ref{subsec: probsetup} that the estimation $\hat{y}_t$ is given by a weighted average of the states of the simulated particles,
\begin{align}\label{eq:weighted_state}
\hat{y}_t = \frac{\sum_{i=1}^N w_{t}^{(i)} x_t^{(i)}}{\sum_{i=1}^N w_{t}^{(i)}}.
\end{align}
and the estimation $\tilde{y}_t$ is given by the posterior mean of $x^*_t$ given observations $o_{0:t}$,
\begin{align}\label{eq:posterior_state}
\tilde{y}_t = \frac{\int_{x'_{1:t}\in \mathcal{X}^t}\prod_{s=1}^t \P\left[o^*_s~|~x'_s\right] x'_t\d \rho_t(x'_{1:t})}{\int_{x'_{1:t}\in \mathcal{X}^t}\prod_{s=1}^t \P\left[o^*_s~|~x'_s\right] \d \rho_t(x'_{1:t})}.
\end{align}  
By Lemma~\ref{lem:decomposition}, we know that to estimate $x_t$ and $x_t^*$, it is enough to estimate $\xi_{0:t-1}$. We can further show that the estimators $\hat{y}_t$ and $\tilde{y}_t$ can be written as a function of estimators $\hat{\xi}_{t,0:t-1}$ and $\tilde{\xi}_{t,0:t-1}$, past actions $\hat{u}_{0:t-1}$ and $u^*_{0:t-1}$, and the initial state $x_0$. The estimator $\hat{\xi}_{t,0:t-1}$ is given by a weighted average of the noise of the particles, $\xi_{0:t-1}^{(1)},...,\xi_{0:t-1}^{(N)}$, similar to \eqref{eq:weighted_state}. The estimator $\tilde{\xi}_{t,0:t-1}$ is given by the posterior mean of the noise given observations, similar to \eqref{eq:posterior_state}. Since $\hat{u}_{0:t-1}$ and ${u}_{0:t-1}$ are determined by $\hat{y}_{0:t-1}$ and $\tilde{y}_{t-1}$, to show that $\hat{y}_t$ is close to $\tilde{y}_t$, it is enough to show that $\hat{\xi}_{s,0:s-1}$ is close to $\tilde{\xi}_{s, 0:s-1}$ in all rounds $s = 1,...,t-1$. We can show the following concentration bound. 
\begin{restatable}[Particle Concentration]{lemma}{concentration}
	\label{lem:concentration}
	Let $M:=\sqrt{ \frac{d}{m}(1+2\sqrt{\log\beta'/d} +2\log\beta'/d )}$ for some $\beta'  > 1$. At time $t\in[T]$, for each $s=0,..,t-1$, we have for any $\beta\leq \frac{1}{2}$,
	\begin{align*}
	\| \hat{\xi}_{t,s} - \tilde{\xi}_{t,s} \| \leq 4\beta M ,
	\end{align*}
	holds with probability at least 
	\begin{align*}
	1 - (d+1) \exp ( -  N\beta^{2}\gamma_t / 3 ) - N\exp(-\beta').
	\end{align*}
\end{restatable}
Lemma~\ref{lem:concentration} studied how the particle approximation concentrates in one time step. We next discuss how the error of approximation in one time step can affect the actions in the future and further affect the long-run reward of the process. 

From Lemma~\ref{lem:decomposition}, it is easy to see that the distance between $x_t$ and $x_t^*$ is determined by the distance between actions in the past time steps, $\hat{u}_{0:t-1}$ and $u^*_{0:t-1}$. It is easy to see that the distance between $x_t$ and $x_t^*$ is determined by the distance between actions in the past time steps, $\hat{u}_{0:t-1}$ and $u^*_{0:t-1}$. 
\begin{restatable}{lemma}{statedist}
	\label{lem:state_dist}
	At time $t$, 
	\begin{align*}
	x_t - x^*_t = \sum_{s=0}^{t-1}\left( \prod_{s'=s+1}^{t-1} A_{s'}\right) B_{s} \left( \hat{u}_{s}  - u^*_s \right).
	\end{align*}
\end{restatable}
\begin{proof}
	The proof follows directly from the problem setup. 
\end{proof}
 Then, to bound the distance between the states $x_t$ and $x_t^*$, it suffices to bound the distance between the actions $\hat{u}_{0:t-1}$ and $u^*_{0:t-1}$. We show the bound in Lemma~\ref{lem:action_dist2}.

\begin{restatable}{lemma}{actiondist}
	\label{lem:action_dist2}
	Assume that $\max_{0\leq s<t\leq T} \| \hat{\xi}_{t,s}-\tilde{\xi}_{t,s} \| = \epsilon $. At time $t$, we can show the following bounds on $\| \hat{u}_{t} - u_{t}^{*} \| $.
	\begin{itemize}
		\item Under Assumptions~\ref{asmp:subgaussian},~\ref{asmp:reward_lip} and~\ref{asmp:g_lip}, for $t\in[T]$, let $\Sigma_a^{(t)} =1 +  C_a\sum_{s=0}^{t-2} \rho_a^s$ and  $\Sigma_{ab}^{(t-1)} =\sum_{s=0}^{t-2} (C_a + C_bL_g)^{s}$.  Then, we have
		\begin{align*}
		\| \hat{u}_{t} - u_{t}^{*} \| 
		\leq &  L_g\Sigma_a^{(t)} \left( 1 +  L_g  C_b \Sigma_{ab}^{(t-1)}  \right)  \cdot \epsilon .
		\end{align*}
		\item  Under Assumptions~\ref{asmp:subgaussian},~\ref{asmp:reward_lip} and~\ref{asmp:g_linear}, for $t\in[T]$, let $\Sigma_a^{(t)} =1 +  C_a\sum_{s=0}^{t-2} \rho_a^s$ and $\bar{\Sigma}_{ab}^{(t-1)} =1 + C_{ab}\sum_{s=0}^{t-3} \rho_{ab}^{s}$. Then, we have
		\begin{align*}
		\| \hat{u}_{t} - u_{t}^{*} \| 
		\leq & ~ L_g\Sigma_a^{(t)}\left( 1 + C_{bg}\bar{\Sigma}_{ab}^{(t-1)}\right) \cdot \epsilon.
		\end{align*}
	\end{itemize}
\end{restatable}

Finally, we use Lemma~\ref{lem:action_dist2} to prove our main results Theorem~\ref{thm:nonlinear} and Theorem~\ref{thm:main_linear}.
\begin{proof}[Proof of Theorem~\ref{thm:nonlinear} and Theorem~\ref{thm:main_linear}]
	We state the proof for the Lipschitz $g$ case here. The proof for linear $g$ follows the same steps.
	We first show the number of particles needed so that the estimation of the noise, $\xi_{t}$, in a single round is accurate. 
	If
	\begin{align}
	\label{eq:num_particle2}
	N = \Omega ( \beta^{-2}p^{-1} \log (dT /\delta) ), 
	\end{align}
	then
	\begin{align*}
	(d+1) \cdot \exp(-N\beta^2\gamma_t/3 ) \leq & ~ (d+1) \cdot \exp(-N\beta^2p/3 )  \\
	\leq  ~ \delta / (2T^2 ) ,
	\end{align*}
	where the first inequality follows from $$\gamma_t = \P_{O^*_{1:t}}[o_{1:t}^*|u_{0:t}^*,x_0] = \P_{O_{1:t}}[o_{1:t}|\hat{u}_{0:t},x_0] \geq p.$$ 
	Let $M:=\sqrt{ \frac{d}{m}(1+2\sqrt{\log\beta'/d} +2\log\beta'/d )}$. If we choose $\beta' =  \log (2T^2N/\delta)$ and $\beta = \epsilon /(4MT)$, by Lemma~\ref{lem:concentration}, with success probability at least 
	\begin{align*}
	1 - \sum_{t=1}^T\sum_{s=0}^{t-1} \delta/(2T^2) - \sum_{t=1}^T\sum_{s=0}^{t-1} \delta/(2T^2)  \geq 1-\delta,
	\end{align*}
	we have for all time step $t = 1,...,T$ and $s = 0,..,t-1$, 
	\[
	\| \hat{\xi}_{t,s}-\tilde{\xi}_{t,s} \| \leq 4\beta M=  \epsilon /T .
	\]

	Next, we bound the distance between actions in the two processes. By Lemma $\ref{lem:action_dist2}$, for any $t = 1,...,T$,
	\begin{align}
	\label{eq:u_dist2}
	\left\| \hat{u}_{t} - u_{t}^{*} \right\|   \leq & ~
	 L_g\Sigma_a^{(t)} \left( 1 +  L_g  C_b \Sigma_{ab}^{(t-1)}  \right)  \cdot  \frac \epsilon T.
	\end{align}
	The second inequality follows from our assumption.
	By Lemma~\ref{lem:state_dist} and our assumptions, we can further bound the distance between the states of the two processes as
	\begin{align}
	\label{eq:x_dist2}
	\| x_{t} - x_{t}^{*} \| 
	= & ~\left \|\sum_{s=0}^{t-1}\left( \prod_{s'=s+1}^{t-1} A_{s'}\right) B_s ( \hat{u}_{s}  - u^*_s ) \right\| \notag \\
	\leq  & ~ C_b  \left( \left\| \hat{u}_{t-1}  - u^*_{t-1}  \right\| + C_a \sum_{s=0}^{t-2} \rho_a^s  \left\| \hat{u}_{s}  - u^*_s  \right\|\right) \notag  \\
	\leq & ~ C_b \Sigma_a^{(t)}\cdot L_g \Sigma_a^{(t)} \left( 1 +  L_g  C_b \Sigma_{ab}^{(t-1)}  \right)  \cdot \frac \epsilon T, 
	\end{align}
    Thus, combining \eqref{eq:u_dist2} and \eqref{eq:x_dist2}, we can get for any $L_r$-Lipschitz reward function $r_T$,
	\begin{align*}
	&~ r_T(x_{1:T}, \hat{u}_{0:T-1}) - r_T(x^{*}_{1:T}, u^{*}_{0:T-1}) \\
	\leq &~ \sum_{t=1}^T  L_r \| x_{t}-x_{t}^{*} \| + \sum_{t=1}^T  L_r \| \hat{u}_{t-1} - u_{t-1}^{*} \| \\
	\leq & ~ L_rL_g \Sigma_a^{(T)} \left( 1 + C_b\Sigma_a^{(T)}\right)\left( 1 +L_g C_b\Sigma_{ab}^{(T-1)}\right)    \epsilon ,
	\end{align*}
	where the first step follows from $r_T$ is $L_r$-Lipschitz and the second step follows from \eqref{eq:u_dist2} and \eqref{eq:x_dist2}.
	
	Plugging $\beta^2 =  \epsilon^2 /(16T^2M^2) = \tilde{\Theta}(\epsilon^2 T^{-2}d^{-1}m)$ into \eqref{eq:num_particle2}, the number of particles needed is 
	\begin{align*}
	N = & ~ \tilde{O} ( \beta^{-2}p^{-1}  )
	=  \tilde{O} ( T^2dm^{-1} \epsilon^{-2} p^{-1}  ),
	\end{align*}
	which completes the proof. 
	Similarly, we can also show that the number of particles needed for linear $g$ so that 
	\begin{align*}
	&~ r_T(x_{1:T}, \hat{u}_{0:T-1}) - r_T(x^{*}_{1:T}, u^{*}_{0:T-1})  \\
	\leq & ~   L_rL_g \Sigma_a^{(T)} \left( 1 + C_b\Sigma_a^{(T)}\right)\left( 1 +C_{bg}\bar{\Sigma}_{ab}^{(T-1)}\right)    \epsilon
	\end{align*}
	is
	$$N =  \tilde{O} ( T^2dm^{-1} \epsilon^{-2} p^{-1}  ).$$
\end{proof}

\section{Lower Bound}
\label{sec:lowerbound}
In this section, we show that Algorithm~\ref{alg:pf} has particle complexity with at least a linear dependence on the inverse of the likelihood of observation, $1/p$.
 We note that it is possible for $1/p$ to depend exponentially on the number of time step $T$ in some processes. However, we are able to show that it is necessary for the particle complexity to depend on $1/p$. Precisely, in Theorem~\ref{thm:main_linear}, we show that we need $O(1/p)$ particles to approximate the whole process well. We show in this section that the upper bound $O(1/p)$ is tight.  

We consider the following process of dimension $d = 1$. We start from the initial state $x_0 = 0$. The total number of time steps is $T$. At each time step $t= 0,...,T-1$, let the transition matrices $A_t = 1$ and $B_t = 0$.  

Let $\bm{\delta}(\cdot)$ be the standard Dirac Delta function such that $\bm{\delta}(x-a) = 0$ for $x\neq a$ and $\int_{a-\eps}^{a+\eps}\bm{\delta}(x-a) \d x = 1$ for $\eps> 0$. Then, let the density function of the transformation noise $\xi_t$ be given by 
\begin{align*}
	\mu_t( \xi) = \frac{1}{2}\bm{\delta}(\xi - 1) + \frac{1}{2}\bm{\delta}(\xi + 1). 
\end{align*}
for $t = 0,...,T-1$. 
Finally, for all time steps $t = 1,..., T$, let the observation matrix $C_t = 1$ and the observation noise $\zeta_t$ be always $0$, i.e., $\eta_t(\zeta) = \bm{\delta}(\zeta)$. 

Now, we consider the observation $o_T = T$. From the way we construct the process, it is clear  that $x_t = t$ for all $t = 1,...,T$. Then, we must have $\xi_t =1$ for all $t = 0,...,T-1$. Moreover, for the observation  $o_T = T$, $p = \P[o_{1:T}|x_0] = 2^{-T}$. We state this formally in Lemma~\ref{lem:lowlem}.
\begin{lemma}
	\label{lem:lowlem}
	If $o_T = T$, then $x_t = t$ for all $t = 1,...,T$ and therefore $\xi_t =1$ for all $t = 0,...,T-1$. Moreover, for the observation $o_T = T$, $p = \P[o_{1:T}|x_0] = 2^{-T}$.
\end{lemma}
\begin{proof}[Proof of Lemma~\ref{lem:lowlem}]
	Since $\zeta_1 = ... = \zeta_T = 0$, $x_t = o_t$ for all $t = 1,..., T$.
	
	Assume for contradiction that there exists some $\xi_s<0$ for some $0< s < T$. Then we have
	\[
	x_T \leq T-1  < T,
	\]
	contradicting $x_T = T$. Thus, we have $x_t = o_t = t$ for all $t = 1,..., T$. Moreover,
	\begin{align*}
		p = & ~\int_{\xi'_{0:T-1}}\P[o_{1:T}|\xi'_{0:T-1},x_0]\d \pi_T{(\xi'_{0:T-1})} \\
		= & ~ \int_{\xi'_{0:T-1}}\prod_{t=0}^{T-1}\eta_{t+1}\left(t+1-\sum_{s=0}^t\xi_s\right) \d \pi_T{(\xi'_{0:T-1})} \\
		= & ~ \Pr[\xi_t > 0, \forall 0\leq t\leq T-1]\\
		= & ~ 2^{-T}.
	\end{align*}
	The second step follows from the definition of our process. The third step follows from $$\prod_{t=0}^{T-1}\eta_{t+1}\left(t+1-\sum_{s=0}^t\xi_s\right) =0$$ if there exists some $\xi_t < 0$. The last step follows from $\xi_t > 0$ with probability $1/2$, 
\end{proof}

Next, we show that if we do not simulate enough particles, then with high probability, there will not exist any particle $i \in[N]$ that has $\xi_t^{(i)} > 0 $ for all $t = 1,...,T$. Then, all particles will have weight $w_T^{(i)} = 0$ at time step $T$. 
\begin{theorem}
	\label{thm:lower_bound}
	Suppose the number of simulated particles $N\leq 1/(2kp)$ for some $k > 1$. For $i\in[N]$, let $I_i$ be the indicator random variable of the event that $\xi_t^{(i)} > 0 $ for all $t = 1,...,T$. Then we have
$
	\Pr\left[\sum_{i=1}^N I_i \geq 1 \right] \leq \frac{1}{k}.
$
\end{theorem}
\begin{proof}[Proof of Theorem~\ref{thm:lower_bound}]
	Since for each $i\in[N]$, $I_i = 0$ with probability $1-2^{-T} = 1-p$ and $I_i = 1$ with probability $2^{-T} = p$, 
	\begin{align*}
		\Pr \left[ \sum_{i=1}^N I_i \geq 1 \right] \leq & ~ \sum_{i=1}^N \Pr\left[I_i = 1\right] = ~ \frac{1}{k},
	\end{align*}
	which completes the proof. 
\end{proof}

By Theorem~\ref{thm:lower_bound}, to avoid all the weights of particle going to zero after $T$ time steps with high probability, we need to simulate at least $\Omega{(1/p)}$ particles.

\section{Conclusion}
\label{sec:dis}
This paper gives the first quantitative analysis of using particle filtering for planning over latent states. 
We also demonstrate the conditions in our theorem are necessary.
In the following, we list some open problems for future study.

\paragraph{Optimal Particle Complexity} 
A natural interesting theoretical problem is, under the assumptions in Section~\ref{sec:result}: \emph{what is the minimal number of particles needed to find a near-optimal planning policy?}
Note the standard particle filtering algorithm (cf. Algorithm~\ref{alg:pf}) is only one approach that uses particles.
One can design more advanced algorithms that operate on these particles with smaller particle complexity.
For example, particle filtering resampling has shown to outperform standard particle filtering algorithm~\citep{kitagawa1993monte}, and it is possible that this approach also admits theoretical benefits.
On the other hand, proving particle complexity lower bound will also improve our understanding on methods based on particles in general.
We believe designing an algorithm that achieve optimal particle complexity will have impact in both theory and practice.

 \paragraph{Learning with Particle Filtering}
Our work assumes the probabilistic models of transition and emission are known.
Recently, a line of work used particle filtering in both training and planning phases~\citep{karkus2018integrating,jonschkowski2018differentiable}.
 While we have analyzed the planning phase, the analysis for training the probabilistic models is more challenging.
 In this problem, one uses particle filtering to explore the state space and collect the data to train probabilistic models.
Characterizing the sample and particle complexity together is an interesting direction to pursue.

\begin{acknowledgements} 
Part of the work was done while SSD, WH, ZL, RS, ZS were participating the Optimization, Statistics, and Theoretical Machine Learning program at  Institute for Advanced Study of Princeton.
SSD was supported by NSF DMS-1638352 and the Infosys Membership. WH and ZL were supported by NSF, ONR, Simons Foundation, Schmidt Foundation, Amazon Research, DARPA and SRC. JW is supported by ARO MURI W911NF-15-1-0479, the Samsung Global Research Outreach Program, Amazon, IBM, and Stanford HAI.
\end{acknowledgements}

\renewcommand{\bibsection}{\subsubsection*{References}}
\bibliography{refs}

\appendix
\onecolumn
\section*{Appendix}
\input{appendix.tex}

\end{document}

%% file: appendix.tex
\section{Probability tools}

\begin{lemma}[Matrix Bernstein, Theorem 6.1.1 in \cite{t15}]\label{lem:matrix_bernstein}
Consider a finite sequence $\{ X_1, \cdots, X_m \} \subset \R^{n_1 \times n_2}$ of independent, random matrices with common dimension $n_1 \times n_2$. Assume that
\begin{align*}
\E[ X_i ] = 0, \forall i \in [m] ~~~ \mathrm{and}~~~ \| X_i \| \leq M, \forall i \in [m] .
\end{align*}
Let $Z = \sum_{i=1}^m X_i$. Let $\mathrm{Var}[Z]$ be the matrix variance statistic of sum:
\begin{align*}
\mathrm{Var} [Z] = \max \left\{ \Big\| \sum_{i=1}^m \E[ X_i X_i^\top ] \Big\| , \Big\| \sum_{i=1}^m \E [ X_i^\top X_i ] \Big\| \right\}.
\end{align*}
Then 
{\small
\begin{align*}
\E[ \| Z \| ] \leq ( 2 \mathrm{Var} [Z] \cdot \log (n_1 + n_2) )^{1/2} +  M \cdot \log (n_1 + n_2) / 3.
\end{align*}
}
Furthermore, for all $t \geq 0$,
\begin{align*}
\Pr[ \| Z \| \geq t ] \leq (n_1 + n_2) \cdot \exp \left( - \frac{t^2/2}{ \mathrm{Var} [Z] + M t /3 }  \right)  .
\end{align*}
\end{lemma}

\section{Proof of Main Result}
\label{sec:proof_main}
We state the complete version of the proofs shown in Section~\ref{sec:proof_idea} in this section. Parts of Section~\ref{subsec:parapprox}, \ref{subsec:actdiff} and \ref{subsec:rewbound} have been stated in Section~\ref{sec:proof_idea}, we restate here for completeness. In Section~\ref{subsec:parapprox}, we give a concentration bound on the particle approximation of the latent state. In Section~\ref{subsec:actdiff}, we study how the error of inference in each round accumulates through the sequential planning process. In Section~\ref{subsec:rewbound}, we put the pieces together to give the upper bound on the number of particles needed so that the long-run rewards of the two processes are close.  We show the proofs of most of the lemmas in this section to Section~\ref{sec:defproof}. 
\subsection{Particle  Concentration}
\label{subsec:parapprox}
We first note that at time $t$, since we know the initial state $x_0$, the transition matrices $A_{0:t-1}$ and $B_{0:t-1}$ and the past actions $\hat{u}_0,...,\hat{u}_{t-1}$, estimating the state $x_t$ is equivalent to estimating $\xi_0,...,\xi_{t-1}$. We show in Lemma~\ref{lem:decomposition} that we can write the states as a function of the initial state, past transformation noise and actions, which follows straightly from our definitions of the processes. 
\decomposition*

Recall from Section~\ref{subsec: probsetup} that the estimation $\hat{y}_t$ is given by a weighted average of the states of the simulated particles,
\begin{align}\label{eq:weighted_state}
\hat{y}_t = \frac{\sum_{i=1}^N w_{t}^{(i)} x_t^{(i)}}{\sum_{i=1}^N w_{t}^{(i)}}.
\end{align}
and the estimation $\tilde{y}_t$ is given by the posterior mean of $x^*_t$ given observations $o_{0:t}$,
\begin{align}\label{eq:posterior_state}
\tilde{y}_t = \frac{\int_{x'_{1:t}\in \mathcal{X}^t}\prod_{s=1}^t \P\left[o^*_s~|~x'_s\right] x'_t\d \rho_t(x'_{1:t})}{\int_{x'_{1:t}\in \mathcal{X}^t}\prod_{s=1}^t \P\left[o^*_s~|~x'_s\right] \d \rho_t(x'_{1:t})}.
\end{align}  
By Lemma~\ref{lem:decomposition}, we know that to estimate $x_t$ and $x_t^*$, it is enough to estimate $\xi_{0:t-1}$. Surprisingly, we can further show that the estimators $\hat{y}_t$ and $\tilde{y}_t$ can be written as a function of estimators $\hat{\xi}_{t,0:t-1}$ and $\tilde{\xi}_{t,0:t-1}$, past actions $\hat{u}_{0:t-1}$ and $u^*_{0:t-1}$, and the initial state $x_0$. The estimator $\hat{\xi}_{t,0:t-1}$ is given by a weighted average of the noise of the particles, $\xi_{0:t-1}^{(1)},...,\xi_{0:t-1}^{(N)}$, similar to \eqref{eq:weighted_state}. The estimator $\tilde{\xi}_{t,0:t-1}$ is given by the posterior mean of the noise given observations, similar to \eqref{eq:posterior_state}. We show this formally in Lemma~\ref{lem:error_est}.

\begin{restatable}{lemmma}{errorest}
	\label{lem:error_est}
	At time $t \in [T]$, for any $s=0,..,t-1$, if we estimate $\xi_s$ as $\hat{\xi}_{t,s}$, given by,
	\begin{align*}
	\hat{\xi}_{t,s} = \frac{\sum_{i=1}^N w_{t}^{(i)} \xi_s^{(i)}}{\sum_{i=1}^N w_{t}^{(i)}},
	\end{align*}
	$\hat{y}_t $ can be written as
	\begin{align*}
	\hat{y}_t =  \sum_{s=0}^{t-1} \prod_{s'=s+1}^{t-1} A_{s'} \cdot \left(\hat{\xi}_{t,s} + B_{s}\cdot \hat{u}_{s}\right) + \prod_{s=0}^{t-1} A_{s}\cdot x_0.
	\end{align*}
	Moreover, let $\pi_t(\xi_{0:t-1})$ be the distribution of $\xi_{0:t-1}$, given by the density $\prod_{s=0}^{t-1}\eta_s(\xi_s)$. If we estimate $\xi_s$ as $\tilde{\xi}_{t,s}$, given by,
	\begin{align*}
	\tilde{\xi}_{t,s} = \frac{  \int_{\xi'_{0:t-1}} \P\left[o^*_{1:t} ~|~ \xi'_{0:t-1},u^*_{0:t-1},x_0\right]\xi'_{s}\d\pi_t(\xi'_{0:t-1}) }{ \int_{\xi'_{0:t-1}} \P\left[o^*_{1:t} ~|~ \xi'_{0:t-1},u^*_{0:t-1},x_0\right]\d\pi_t(\xi'_{0:t-1}) },
	\end{align*}
	$\tilde{y}_t$ can be written as
	\begin{align*}
	\tilde{y}_t =  \sum_{s=0}^{t-1} \prod_{s'=s+1}^{t-1} A_{s'} \cdot \left(\tilde{\xi}_{t,s} + B_{s}\cdot u^*_{s}\right) + \prod_{s=0}^{t-1} A_{s}\cdot x_0.
	\end{align*}
\end{restatable}
By Lemma~\ref{lem:error_est}, since $\hat{u}_{0:t-1}$ and ${u}_{0:t-1}$ are determined by $\hat{y}_{0:t-1}$ and $\tilde{y}_{t-1}$, to show that $\hat{y}_t$ is close to $\tilde{y}_t$, it is enough to show that $\hat{\xi}_{s,0:s-1}$ is close to $\tilde{\xi}_{s, 0:s-1}$ in all rounds $s = 1,...,t-1$. In this section, we focus on showing how accurately $\xi_{t,0:t-1}$ can approximate $\tilde{\xi}_{t, 0:t-1}$ in one time step. We postpone the discussion of how the error of this approximation accumulates through the process to Section~\ref{subsec:actdiff}. 

To see how $\hat{\xi}_{t,s}$ can approximate $\tilde{\xi}_{t,s} $, we study the numerator and the denominator of $\hat{\xi}_{t,s}$ and $\tilde{\xi}_{t,s}$ seperately. To simplify our notation, we make the following definition. 

\begin{definition}
	We define the scalar $\gamma_t\in R$ and vector $\Gamma_{t,s} \in \R^d$ for any time $0\leq s< t\leq T$  as follows:
	\begin{align*}
	\gamma_t = & ~  \int_{\xi'_{0:t-1}} \P\left[o^*_{1:t} ~|~ \xi'_{0:t-1},u^*_{0:t-1},x_0\right]\d\pi_t(\xi'_{0:t-1}), \\
	\Gamma_{t,s} = & ~ \int_{\xi'_{0:t-1}} \P\left[o^*_{1:t} ~|~ \xi'_{0:t-1},u^*_{0:t-1},x_0\right]\xi'_{s}\d\pi_t(\xi'_{0:t-1}) .
	\end{align*}
\end{definition}

We can further show that since we couple our two processes using the same noise, the posterior mean of the noise given observations, $o_{1:t}$, and actions $\hat{u}_{0:t-1}$ in the approximate process is the same as that given observations, $o^*_{1:t}$, and actions $u^*_{0:t-1}$ in the ideal process. 
\begin{restatable}{lemmma}{comparison}
	\label{lem:comparison}	
	For any time $0\leq s< t\leq T$ ,
	\begin{align*}
	\gamma_t = & ~  \int_{\xi'_{0:t-1}} \P\left[o_{1:t} ~|~ \xi'_{0:t-1},\hat{u}_{0:t-1},x_0\right]\d\pi_t(\xi'_{0:t-1}), \\
	\Gamma_{t,s} = & ~ \int_{\xi'_{0:t-1}} \P\left[o_{1:t} ~|~ \xi'_{0:t-1},\hat{u}_{0:t-1},x_0\right]\xi'_{s}\d\pi_t(\xi'_{0:t-1}) .
	\end{align*}
\end{restatable}
Then, to show that the particle approximation, $\hat{\xi}_{t,s}$, is close to $\tilde{\xi}_{t,s} $, it is enough to show that $\hat{\xi}_{t,s}$ concentrates around the posterior mean of $\xi_{t,s}$. We show the relationship between the accuracy of particle approximation and the number of particles, $N$, in the following lemma. 
\concentration*
We defer the proof of Lemma~\ref{lem:error_est}, Lemma~\ref{lem:comparison} and Lemma~\ref{lem:concentration} to Appendix~\ref{sec:defproof}.

\subsection{Error Accumulation}
\label{subsec:actdiff}

In Section~\ref{subsec:parapprox}, we studied how the particle approximation concentrates in one time step. In this Section, we discuss how the error of approximation in one time step can affect the actions in the future and further affect the long-run reward of the process. 

Lemma~\ref{lem:decomposition} shows that the states of the two processes, $x_t$ and $x_t^*$, at time step $t$, can be written as 
\begin{align*}
x_t = \sum_{s=0}^{t-1} \prod_{s'=s+1}^{t-1} A_{s'} \cdot \left(\xi_{s} + B_{s}\cdot \hat{u}_{s}\right) + \prod_{s=0}^{t-1} A_{s}\cdot x_0,\\
x^*_t = \sum_{s=0}^{t-1} \prod_{s'=s+1}^{t-1} A_{s'} \cdot \left(\xi_{s} + B_{s}\cdot {u}^*_{s}\right) + \prod_{s=0}^{t-1} A_{s}\cdot x_0.
\end{align*}
It is easy to see that the distance between $x_t$ and $x_t^*$ is determined by the distance between actions in the past time steps, $\hat{u}_{0:t-1}$ and $u^*_{0:t-1}$. 
\statedist*
The actions $\hat{u}_s$ and $u^*_s$ at $s=1,...,t-1$ is determined by the state estimations, $\hat{y}_s$ and $\tilde{y}_s$,
\begin{align*}
\hat{u}_{s} =g(\hat{y}_t) \quad \text{and} \quad u^*_s = g(\tilde{y}_s).
\end{align*}
Moreover, by Lemma~\ref{lem:error_est}, at time step $t$,
\begin{align*}
\hat{y}_t =  \sum_{s=0}^{t-1} \prod_{s'=s+1}^{t-1} A_{s'} \cdot \left(\hat{\xi}_{t,s} + B_{s}\cdot \hat{u}_{s}\right) + \prod_{s=0}^{t-1} A_{s}\cdot x_0,\\
\tilde{y}_t =  \sum_{s=0}^{t-1} \prod_{s'=s+1}^{t-1} A_{s'} \cdot \left(\tilde{\xi}_{t,s} + B_{s}\cdot u^*_{s}\right) + \prod_{s=0}^{t-1} A_{s}\cdot x_0.
\end{align*}
which shows that  $\hat{y}_t$ and $\tilde{y}_s$ are in turn determined by the past actions. 
Thus, the key step of bounding the error accumulation is bounding the distance between the actions in the two processes. We show the upper bound on the action distance in the following lemma. 
\actiondist*

We defer the proof of Lemma~\ref{lem:action_dist2} to Appendix~\ref{sec:defproof}. Lemma~\ref{lem:state_dist} and Lemma~\ref{lem:action_dist2} together show that we can bound the distance between the states and the action of the two processes in terms of the accuracy of the particle approximation of transformation noise $\xi_{t}$. 

\subsection{Bound on Reward Difference}
\label{subsec:rewbound}
In this section, we combine the results from Section~\ref{subsec:parapprox} and Section~\ref{subsec:actdiff} to show an upper bound on the number of particle needed so that the rewards of the two processes are close. Lemma~\ref{lem:concentration} upper bounds the number of particles needed so that the particle approximation of the noise $\xi_{t}$ is accurate. Lemma~\ref{lem:state_dist} and Lemma~\ref{lem:action_dist2}  show that the actions and the states of the two processes are close if the particle approximation is accurate. Then, for reward function that depends on states and actions, we can combine these results to upper bound the number of particles that can guarantee the rewards of the two processes are close. We state our main result in Theorem~\ref{thm:main_linear} and show the proof below.
\begin{proof}[Proof of Theorem~\ref{thm:nonlinear} and Theorem~\ref{thm:main_linear}]
	We state the proof for the Lipschitz $g$ case here. The proof for linear $g$ follows the same steps.
	We first show the number of particles needed so that the estimation of the noise, $\xi_{t}$, in a single round is accurate. 
	If
	\begin{align}
	\label{eq:num_particle}
	N = \Omega ( \beta^{-2}p^{-1} \log (dT /\delta) ), 
	\end{align}
	then
	\begin{align*}
	(d+1) \cdot \exp(-N\beta^2\gamma_t/3 ) \leq & ~ (d+1) \cdot \exp(-N\beta^2p/3 ) 
	\leq  ~ \delta / (2T^2 ) ,
	\end{align*}
	where the first inequality follows from $$\gamma_t = \P_{O^*_{1:t}}[o_{1:t}^*|u_{0:t}^*,x_0] = \P_{O_{1:t}}[o_{1:t}|\hat{u}_{0:t},x_0] \geq p.$$ 
	Let $M:=\sqrt{ \frac{d}{m}(1+2\sqrt{\log\beta'/d} +2\log\beta'/d )}$. If we choose $\beta' =  \log (2T^2N/\delta)$ and $\beta = \epsilon /(4MT)$, by Lemma~\ref{lem:concentration}, with success probability at least 
	\begin{align*}
	1 - \sum_{t=1}^T\sum_{s=0}^{t-1} \delta/(2T^2) - \sum_{t=1}^T\sum_{s=0}^{t-1} \delta/(2T^2)  \geq 1-\delta,
	\end{align*}
	we have for all time step $t = 1,...,T$ and $s = 0,..,t-1$, 
	\[
	\| \hat{\xi}_{t,s}-\tilde{\xi}_{t,s} \| \leq 4\beta M=  \epsilon /T .
	\]

	Next, we bound the distance between actions in the two processes. By Lemma $\ref{lem:action_dist2}$, for any $t = 1,...,T$,
	\begin{align}
	\label{eq:u_dist}
	\left\| \hat{u}_{t} - u_{t}^{*} \right\|   \leq & ~
	 L_g\Sigma_a^{(t)} \left( 1 +  L_g  C_b \Sigma_{ab}^{(t-1)}  \right)  \cdot  \frac \epsilon T.
	\end{align}
	The second inequality follows from our assumption.
	By Lemma~\ref{lem:state_dist} and our assumptions, we can further bound the distance between the states of the two processes as
	\begin{align}
	\label{eq:x_dist}
	\| x_{t} - x_{t}^{*} \| 
	= & ~\left \|\sum_{s=0}^{t-1}\left( \prod_{s'=s+1}^{t-1} A_{s'}\right) B_s ( \hat{u}_{s}  - u^*_s ) \right\| \notag 
	\leq  ~ C_b  \left( \left\| \hat{u}_{t-1}  - u^*_{t-1}  \right\| + C_a \sum_{s=0}^{t-2} \rho_a^s  \left\| \hat{u}_{s}  - u^*_s  \right\|\right) \notag  \\
	\leq & ~ C_b \Sigma_a^{(t)}\cdot L_g \Sigma_a^{(t)} \left( 1 +  L_g  C_b \Sigma_{ab}^{(t-1)}  \right)  \cdot \frac \epsilon T, 
	\end{align}
    Thus, combining \eqref{eq:u_dist} and \eqref{eq:x_dist}, we can get for any $L_r$-Lipschitz reward function $r_T$,
	\begin{align*}
	&~ r_T(x_{1:T}, \hat{u}_{0:T-1}) - r_T(x^{*}_{1:T}, u^{*}_{0:T-1})
	\leq ~ \sum_{t=1}^T  L_r \| x_{t}-x_{t}^{*} \| + \sum_{t=1}^T  L_r \| \hat{u}_{t-1} - u_{t-1}^{*} \| \\
	\leq & ~ L_rL_g \Sigma_a^{(T)} \left( 1 + C_b\Sigma_a^{(T)}\right)\left( 1 +L_g C_b\Sigma_{ab}^{(T-1)}\right)    \epsilon ,
	\end{align*}
	where the first step follows from $r_T$ is $L_r$-Lipschitz and the second step follows from \eqref{eq:u_dist} and \eqref{eq:x_dist}.
	
	Plugging $\beta^2 =  \epsilon^2 /(16T^2M^2) = \tilde{\Theta}(\epsilon^2 T^{-2}d^{-1}m)$ into \eqref{eq:num_particle}, the number of particles needed is 
	\begin{align*}
	N = & ~ \tilde{O} ( \beta^{-2}p^{-1}  )
	=  \tilde{O} ( T^2dm^{-1} \epsilon^{-2} p^{-1}  ),
	\end{align*}
	which completes the proof. 
	Similarly, we can also show that the number of particles needed for linear $g$ so that 
	\begin{align*}
	&~ r_T(x_{1:T}, \hat{u}_{0:T-1}) - r_T(x^{*}_{1:T}, u^{*}_{0:T-1})  \leq  ~   L_rL_g \Sigma_a^{(T)} \left( 1 + C_b\Sigma_a^{(T)}\right)\left( 1 +C_{bg}\bar{\Sigma}_{ab}^{(T-1)}\right)    \epsilon
	\end{align*}
	is
	$$N =  \tilde{O} ( T^2dm^{-1} \epsilon^{-2} p^{-1}  ).$$
\end{proof}

\subsection{Deferred Proofs}
\label{sec:defproof}

\errorest*

\begin{proof}
	By Lemma~\ref{lem:decomposition}, for every particle $i\in[N]$,
	\begin{align*}
	x_t^{(i)}  = \sum_{s=0}^{t-1} \prod_{s'=s+1}^{t-1} A_{s'} \left(\xi_{s}^{(i)}+ B_{s}\cdot \hat{u}_{s}\right) + \prod_{s=0}^{t-1} A_{s}\cdot x_0.
	\end{align*}
	Then,
	\begin{align*}
	\hat{y}_t & = \frac{\sum_{i=1}^N w_{t}^{(i)} x_t^{(i)}}{\sum_{i=1}^N w_{t}^{(i)}}, \\
	& =  \sum_{s=0}^{t-1} \prod_{s'=s+1}^{t-1} A_{s'} \cdot \left(\frac{\sum_{i=1}^N w_{t}^{(i)} \xi_t^{(i)}}{\sum_{i=1}^N w_{t}^{(i)}}+ B_{s}\cdot \hat{u}_{s}\right) + \prod_{s=0}^{t-1} A_{s}\cdot x_0\\
	& =  \sum_{s=0}^{t-1} \prod_{s'=s+1}^{t-1} A_{s'} \cdot \left(\hat{\xi}_{s} + B_{s}\cdot \hat{u}_{s}\right) + \prod_{s=0}^{t-1} A_{s}\cdot x_0.
	\end{align*}
	Similarly, by Lemma~\ref{lem:decomposition} and the definition of $\rho_t$, 
	\begin{align*}
	\tilde{y}_t & = ~\frac{\int_{x'_{1:t}\in \mathcal{X}^t}\prod_{s=0}^t \P\left[o^*_s~|~x'_s\right] x'_t\d \rho_t(x'_{1:t})}{\int_{x'_{1:t}\in \mathcal{X}^t}\prod_{s=0}^t \P\left[o^*_s~|~x'_s\right] \d \rho_t(x'_{1:t})} \\
	& = ~\frac{  \int_{\xi'_{0:t-1}} \P\left[o^*_{1:t} ~|~ \xi'_{0:t-1},u^*_{0:t-1},x_0\right] }{ \int_{\xi'_{0:t-1}} \P\left[o^*_{1:t} ~|~ \xi'_{0:t-1},u^*_{0:t-1},x_0\right]\d\pi_t(\xi'_{0:t-1}) } \left[ \sum_{s=0}^{t-1} \prod_{s'=s+1}^{t-1} A_{s'} \cdot \left({\xi}'_{s} + B_{s}\cdot {u}^*_{s}\right) + \prod_{s=0}^{t-1} A_{s}\cdot x_0\right] \d\pi_t(\xi'_{0:t-1}) \\
	& =   \sum_{s=0}^{t-1} \prod_{s'=s+1}^{t-1} A_{s'} \cdot \left(\tilde{\xi}_{t,s} + B_{s}\cdot u^*_{s}\right) + \prod_{s=0}^{t-1} A_{s}\cdot x_0.
	\end{align*}
\end{proof}
\comparison*
\begin{proof}
	For any $t\in[T]$, we have 
	\begin{align*}
	& ~ \P_{O^*_{1:t}}\left[o_{1:t}^{*} ~|~ \xi'_{0:t-1},u_{0:t-1}^{*},x_0\right]\\
	= & ~ \prod_{t'=1}^{t}\P_{O^*_{1:t}}\left[o_{t'}^{*} ~|~ \xi'_{0:t'-1},u_{0:t'-1}^{*},x_0\right]\\
	= & ~ \prod_{t'=1}^{t}\eta_{t'}\left(\left[ \sum_{s=0}^{t'-1} \prod_{s'=s+1}^{t'-1} A_{s'} \cdot \left({\xi}_{s} + B_{s}\cdot \hat{u}_{s}\right) + \prod_{s=0}^{t'-1} A_{s}\cdot x_0 +\zeta_{t'}\right]-\left[ \sum_{s=0}^{t'-1} \prod_{s'=s+1}^{t'-1} A_{s'} \cdot \left(\hat{\xi'}_{s} + B_{s}\cdot \hat{u}_{s}\right) + \prod_{s=0}^{t'-1} A_{s}\cdot x_0\right]\right)\\
	= & ~  \prod_{t'=1}^{t}\eta_{t'}\left(\sum_{s=0}^{t'-1} \prod_{s'=s+1}^{t'-1} A_{s'} \cdot \left({\xi}_{s} -\xi'_s\right) +\zeta_{t'}\right)\\
	= & ~ \P_{O_{1:t}}\left[o_{1:t} ~|~ \xi'_{0:t-1},\hat{u}_{0:t-1},x_0\right],
	\end{align*}
	where the third step follows from  Lemma~\ref{lem:decomposition}, so 
	\begin{align*}
	\int_{\xi'_{0:t-1}} \P_{O_{1:t}}\left[o_{1:t} ~|~ \xi'_{0:t-1},\hat{u}_{0:t-1},x_0\right]\d\pi_t(\xi'_{0:t-1}) = & ~  \int_{\xi'_{0:t-1}} \P_{O^*_{1:t}}\left[o^*_{1:t} ~|~ \xi'_{0:t-1},u^*_{0:t-1},x_0\right]\d\pi_t(\xi'_{0:t-1}), \\
	\int_{\xi'_{0:t-1}} \P_{O_{1:t}}\left[o_{1:t} ~|~ \xi'_{0:t-1},\hat{u}_{0:t-1},x_0\right]\xi'_{s}\d\pi_t(\xi'_{0:t-1}) = & ~ \int_{\xi'_{0:t-1}} \P_{O^*_{1:t}}\left[o^*_{1:t} ~|~ \xi'_{0:t-1},u^*_{0:t-1},x_0\right]\xi'_{s}\d\pi_t(\xi'_{0:t-1}) .
	\end{align*}
\end{proof}

\concentration*
\begin{proof}
	We first consider the random variables 
	\begin{align*}
	\P\left[ o_{1:t} ~|~ \xi_{0:t-1}^{(i)},\hat{u}_{0:t-1},x_0 \right]= \prod_{t'=1}^{t}\eta_{t'}\left(\sum_{s=0}^{t'-1} \prod_{s'=s+1}^{t'-1} A_{s'} \cdot \left({\xi}_{s} -\xi'_s\right) +\zeta_{t'}\right),
	\end{align*}
	for $i=1,...,N$. 
	By the way we generate $\xi_{0:t}^{(1)}, \xi_{0:t}^{(2)},...,\xi_{0:t}^{(N)}$,
	\begin{align*}
	\P\left[o_{1:t} ~|~ \xi_{0:t-1}^{(1)},\hat{u}_{0:t-1},x_0\right], \P\left[o_{1:t} ~|~ \xi_{0:t}^{(2)},\hat{u}_{0:t},x_0\right], ...,\P\left[o_{1:t} ~|~ \xi_{0:t-1}^{(N)},\hat{u}_{0:t-1},x_0\right]
	\end{align*}
	are independent. Also, for $i=1,...,N$, by Lemma~\ref{lem:comparison}
	\begin{align*}
	\E\left[\P\left[o_{1:t} ~|~ \xi_{0:t-1}^{(i)},\hat{u}_{0:t-1},x_0\right]\right]= & ~\int_{\xi'_{0:t-1}}\P\left[o_{1:t} ~|~ \xi'_{0:t-1},\hat{u}_{0:t-1},x_0\right]\d\pi_t(\xi'_{0:t-1}) 
	=  ~ \gamma_t.
	\end{align*}
	and 
	\begin{align*}
	\E\left[\P\left[o_{1:t} ~|~ \xi_{0:t-1}^{(i)},\hat{u}_{0:t-1},x_0\right]\xi'_{s}\right]= & ~\int_{\xi'_{0:t-1}}\P\left[o_{1:t} ~|~ \xi'_{0:t-1},\hat{u}_{0:t-1},x_0\right]\xi'_{s}\d\pi_t(\xi'_{0:t-1}) 
	=  ~ \Gamma_{t,s}.
	\end{align*}
	By Lemma~\ref{lem:matrix_bernstein},
	\begin{align*}
	& ~ \Pr\left[\left|\frac{1}{N}\sum_{i=1}^{N}\P\left[o_{1:t} ~|~ \xi_{0:t-1}^{(i)},\hat{u}_{0:t-1},x_0\right]-\gamma_t\right|\geq \beta\gamma_t\right]\\
	\leq & ~ \exp\left(-\frac{N\beta^{2}\gamma_t^2}{2\text{Var}\left[\P\left[o_{1:t} ~|~ \xi_{0:t-1}^{(i)},\hat{u}_{0:t-1},x_0\right]\right]+\frac{2}{3}\max \left|\P\left[o_{1:t} ~|~ \xi_{0:t-1}^{(i)},\hat{u}_{0:t-1},x_0\right]\right|\beta\gamma_t}\right)\\
	\leq & ~ \exp ( -  N\beta^{2}\gamma_t / 3 ) .
	\end{align*}
	where the third step follows from 
	\begin{align*}
	\text{Var}\left[\P\left[o_{1:t} ~|~ \xi_{0:t-1}^{(i)},\hat{u}_{0:t-1},x_0\right]\right] \leq & ~ \E\left[\P^2\left[o_{1:t} ~|~ \xi_{0:t-1}^{(i)},\hat{u}_{0:t-1},x_0\right]\right] \\
	= & ~ \int_{\xi'_{0:t-1}}\P^2\left[o_{1:t} ~|~ \xi'_{0:t-1},\hat{u}_{0:t-1},x_0\right]\d\pi_t(\xi'_{0:t-1}) \\
	\leq & ~ \gamma_t,
	\end{align*}
	and
	\begin{align*}
	\max\left|\P\left[o_{1:t} ~|~ \xi_{0:t-1}^{(i)},\hat{u}_{0:t-1},x_0\right]\right| \leq 1.
	\end{align*}
	Without loss of generality, we assume the noise $\xi$ has mean zero. Since the noise $\| \xi^{(i)}_{s} \|$ is sub-gaussian, with probability at least $ 1 - N\exp(-\beta')$, for all $i\in[N]$, $$\| \xi^{(i)}_{s} \|^2 \leq M^2 = \frac{d}{m}(1+2\sqrt{\log\beta'/d} +2\log\beta'/d ).$$
	Similarly, by Lemma~\ref{lem:matrix_bernstein}, since the noise $\| \xi^{(i)}_{s} \| \leq M$ for all $i \in [N]$, 
	\begin{align*}
	& ~ \Pr\left[\left\| \frac{1}{N}\sum_{i=1}^{N}\P\left[o_{1:t} ~|~ \xi_{0:t-1}^{(i)},\hat{u}_{0:t-1}, x_0\right]\xi_{s}^{(i)}-\Gamma_t\right\| \geq \beta \gamma_t M\right]\\
	\leq & ~ d \cdot \exp\left(-\frac{N\beta^{2}\gamma_t^2M^2}{2\text{Var}\left[\P\left[o_{1:t} ~|~ \xi_{0:t-1}^{(i)},\hat{u}_{0:t-1},x_0\right]\xi_{s}^{(i)}\right]+\frac{2}{3}\max\left\| \P\left[o_{1:t} ~|~ \xi_{0:t-1}^{(i)},\hat{u}_{0:t-1},x_0\right]\xi_{s}^{(i)}\right\| \beta\gamma_t M}\right)\\
	\leq & ~ d \cdot \exp  ( -  N \beta^{2} \gamma_t /  3 ).
	\end{align*}
	where the third step follows from 
	\begin{align*}
	\text{Var}\left[\P\left[o_{1:t} ~|~ \xi_{0:t-1}^{(i)},\hat{u}_{0:t-1},x_0\right]\xi_{s}^{(i)}\right] \leq \gamma_t M^2,
	\end{align*}
	and
	\begin{align*}
	\max\left\| \P\left[o_{1:t} ~|~ \xi_{0:t-1}^{(i)},\hat{u}_{0:t-1},x_0\right]\xi_{s}^{(i)}\right\| \leq M.
	\end{align*}
	Then,  with probability at least 
	\begin{align*}
	1 - (d+1) \exp ( -  N\beta^{2}\gamma_t / 3 ) - N \exp(-\beta'),
	\end{align*}
	we have 
	\begin{align*}
	& ~ \| \hat{\xi}_{t,s}-\tilde{\xi}_{t,s} \|  =  \left\| \hat{\xi}_{t,s} - \frac{\Gamma_{t,s}}{\gamma_t} \right\| \\
	\leq & ~ \max\left\{ \left\| \frac{1}{\left(1-\beta\right)\gamma_t} \sum_{i=1}^{N}w_{t}^{(i)}\xi_{t'}^{(i)} -\frac{\Gamma_{t,s}}{\gamma_t}\right\| ,\left\| \frac{1}{\left(1+\beta\right)\gamma_t } \sum_{i=1}^{N}w_{t}^{(i)}\xi_{t'}^{(i)} -\frac{\Gamma_{t,s}}{\gamma_t}\right\| \right\} \\
	\leq & ~ \max\Bigg\{ \frac{1}{1-\beta}\left\| \frac{\sum_{i=1}^{N}w_{t}^{(i)}\xi_{t'}^{(i)}}{\gamma_t}  -\frac{\Gamma_{t,s}}{\gamma_t}\right\| + \left( \frac{1}{1-\beta}-1\right)\left\|  \frac{\Gamma_{t,s}}{\gamma_t} \right\| ,  \\
	& \frac{1}{1+\beta}\left\| \frac{\sum_{i=1}^{N}w_{t}^{(i)}\xi_{t'}^{(i)} }{\gamma_t} - \frac{\Gamma_{t,s}}{\gamma_t} \right\| + \left(1- \frac{1}{1+\beta}\right) \left\| \frac{\Gamma_{t,s}}{\gamma_t} \right\| \Bigg\} \\
	\leq  & ~ \max \left\{  \frac{\beta M}{1-\beta}+\frac{\beta M}{1-\beta}  , \frac{\beta M}{1+\beta} +\frac{\beta M}{1+\beta}  \right\} \\		
	\leq & ~ 4\beta M
	\end{align*}
	where the first step follows from $| \sum_{i=1}^{N} w_{t}^{(i)} - \gamma_t | \leq \beta \gamma_t$, the second step follows from triangle inequality, the third step follows from $\left\|\sum_{i=1}^{N}w_{t}^{(i)}\xi_{t'}^{(i)}-\Gamma_{t,s}\right\|\leq \beta \gamma_t M$, and 
	\begin{align*}
	\left\| \frac{\Gamma_{t,s}}{\gamma_t}\right\| 
	= & ~ \left\| \frac{\int_{\xi'_{0:t-1}} \P\left[o_{1:t} ~|~ \xi'_{0:t-1},\hat{u}_{0:t-1},x_0\right]\xi'_{s}\d\pi_t(\xi'_{0:t-1}) }{\int_{\xi'_{0:t-1}} \P\left[o_{1:t} ~|~ \xi'_{0:t-1},\hat{u}_{0:t-1},x_0\right]\d\pi_t(\xi'_{0:t-1}) }\right\| 
	\leq  ~ M,
	\end{align*}
	and the last step follows from $\beta\leq \frac{1}{2}$.	
\end{proof}

\actiondist*
\begin{proof}
	When $t=0$, we have $\hat{u}_{0}=u_{0}^{*}=g(x_{0})$, so $\left\| \hat{u}_{0}-u_{0}^{*} \right\|=0$. For $t > 0$, we study the two cases separately. 	

	In the first case, for $L_g$-Lipschitz $g$, at time $t>0$, 
		\begin{align*}
	&~ \| \hat{u}_{t} - u_{t}^{*} \| 
	= ~ \| g( \hat{y}_{t} ) - g ( \tilde{y}_{t} ) \| \\
	= &~ \left\| g\left(  \sum_{s=0}^{t-1} \prod_{s'=s+1}^{t-1} A_{s'} \cdot \left(\hat{\xi}_{t,s} + B_{s}\cdot \hat{u}_{s}\right) + \prod_{s=0}^{t-1} A_{s}\cdot x_0\right) - g\left(  \sum_{s=0}^{t-1} \prod_{s'=s+1}^{t-1} A_{s'} \cdot \left(\tilde{\xi}_{t,s} + B_{s}\cdot u^*_{s}\right) + \prod_{s=0}^{t-1} A_{s}\cdot x_0\right)\right\| \\
	\leq & ~ L_g \cdot \left\| \left(  \sum_{s=0}^{t-1} \prod_{s'=s+1}^{t-1} A_{s'} \cdot \left(\hat{\xi}_{t,s} + B_{s}\cdot \hat{u}_{s}\right) + \prod_{s=0}^{t-1} A_{s}\cdot x_0\right) - \left(  \sum_{s=0}^{t-1} \prod_{s'=s+1}^{t-1} A_{s'} \cdot \left(\tilde{\xi}_{t,s} + B_{s}\cdot u^*_{s}\right) + \prod_{s=0}^{t-1} A_{s}\cdot x_0\right)\right\|  \\
	\leq & L_g \cdot \left\|  \sum_{s=0}^{t-1} \prod_{s'=s+1}^{t-1} A_{s'}  \cdot (\hat{\xi}_{t,s}-\tilde{\xi}_{t,s}) \right\| + L_g \cdot  \left\| \sum_{s=0}^{t-1} \prod_{s'=s+1}^{t-1} A_{s'} B_s \cdot (\hat{u}_{s}-u_{s}^{*})\right\|.
	\end{align*}
	where the first step follows from definitions of $\hat{u}_t$ and $u^*_t$, the second step follows from Lemma~\ref{lem:error_est}, the third step follows from $g$ is $L_g$-Lipschitz and the last step follows from triangle inequality.

	We define $f_t$ and $h_t$ as follows:
	\begin{align*}
	f_{t} = & ~ \left\| \sum_{s=0}^{t-1} \prod_{s'=s+1}^{t-1} A_{s'}  \cdot \left( \hat{\xi}_{t,s}-\tilde{\xi}_{t,s}\right)  \right\|, \\
	h_{t} = & ~ \left\| \sum_{s=0}^{t-1} \prod_{s'=s+1}^{t-1} A_{s'} B_s \cdot \left( \hat{u}_{s}-u_{s}^{*}\right)  \right\|.
	\end{align*}
	Then,
	\begin{align*}
	h_{t}
	\leq & ~ C_ah_{t-1}+ C_b \| \hat{u}_{t-1}-u_{t-1}^{*}\| \\
	\leq & ~ C_a h_{t-1}+ C_b L_g (f_{t-1}+h_{t-1})\\
	\leq & ~  C_b L_g f_{t-1} +(C_a+ C_bL_g)\left(C_a h_{t-2}+ C_b L_g (f_{t-2}+h_{t-2})\right)  \\
	\leq & ~ \cdots \\
	\leq & ~C_b L_g \sum_{s=1}^{t-1}(C_a +C_bL_g)^{t-s-1}f_{s},
	\end{align*}
	The first step follows from definition. The second step follows from $\| \hat{u}_{t-1} - u_{t-1}^{*} \| \leq L_g(f_{t-1} + h_{t-1}).$ The third step and the last step follow from induction.
	Thus, for $L_g$-Lipschitz $g$, 
	\begin{align*}
	\| \hat{u}_{t} - u_{t}^{*} \| \leq & ~L_g\cdot f_t + L_g \cdot C_b L_g \sum_{s=1}^{t-1}(C_a+C_bL_g)^{t-s-1}f_{s} \\
	\leq  & ~  L_g\left(  1 + C_a\sum_{s = 0}^{t-2}\rho_a^{s}\right) \epsilon+ L_g \cdot C_b L_g \sum_{s=1}^{t-1}(C_a+C_bL_g)^{t-s-1}\cdot\left( 1+  C_a\sum_{s' = 0}^{s-2}\rho_a^{s'} \right) \cdot \epsilon  \\
	\leq & ~  L_g\Sigma_a^{(t)} \cdot \epsilon + L_g \cdot C_b L_g \Sigma_{ab}^{(t-1)} \cdot \Sigma_a^{(t)}  \cdot \epsilon 
	= ~ L_g\Sigma_a^{(t)} \left( 1 +  L_g  C_b \Sigma_{ab}^{(t-1)}  \right)  \cdot \epsilon  .
	\end{align*}
	The second step follows from our assumption and the last two steps follow from our definitions of $\Sigma_a^{(t)} $ and $ \Sigma_{ab}^{(t)} $.
	
	In the second case, for linear $g = G$, similarly, we have
	\begin{align*}
	&~  \hat{u}_{t} - u_{t}^{*}  = ~ g( \hat{y}_{t} ) - g ( \tilde{y}_{t} )  \\
	= & ~ G\cdot   \sum_{s=0}^{t-1} \prod_{s'=s+1}^{t-1} A_{s'}  \cdot (\hat{\xi}_{t,s}-\tilde{\xi}_{t,s}) + G \cdot   \sum_{s=0}^{t-1} \prod_{s'=s+1}^{t-1} A_{s'} B_s \cdot (\hat{u}_{s}-u_{s}^{*}).
	\end{align*}
	We define $f_t$ and $h_t$ as follows:
	\begin{align*}
	f_{t} = & ~ \sum_{s=0}^{t-1} \prod_{s'=s+1}^{t-1} A_{s'}  \cdot \left( \hat{\xi}_{t,s}-\tilde{\xi}_{t,s}\right) , \\
	h_{t} = & ~  \sum_{s=0}^{t-1} \prod_{s'=s+1}^{t-1} A_{s'} B_s \cdot \left( \hat{u}_{s}-u_{s}^{*}\right).
	\end{align*}
	then by induction,
	\begin{align*}
	 h_{t}
	= & ~ A_{t-1}h_{t-1} + B_{t-1} \left(\hat{u}_{t-1}-u_{t-1}^{*} \right)\\
	\leq & ~ A_{t-1} h_{t-1}+ B_{t-1}G (f_{t-1}+h_{t-1})\\
	\leq & ~  B_{t-1} G f_{t-1} + (A_{t-1} +B_{t-1} G)\left(A_{t-2} h_{t-2}+ B_{t-2} G (f_{t-2}+h_{t-2})\right)  \\
	\leq & ~ \cdots \\
	\leq & ~\sum_{s=1}^{t-1} \prod_{s' = s+1}^{t-1} (A_{s'} + B_{s'}G) B_{s}G f_{s},
	\end{align*}
	Thus, by our assumptions and the definitions of $\Sigma_a^{(t)} $ and $ \bar{\Sigma}_{ab}^{(t)} $. 
	\begin{align*}
	\| \hat{u}_{t} - u_{t}^{*} \| = & ~\left\Vert G\cdot  f_t + G \cdot  \sum_{s=1}^{t-1} \prod_{s' = s+1}^{t-1} (A_{s'} + B_{s'}G) B_{s}G f_{s} \right\Vert  \\
	\leq & ~L_g\left( 1+ C_a\sum_{s = 0}^{t-2}\rho_a^{s}  \right) \cdot \epsilon  + L_g \cdot C_{bg} \left(1+ C_{ab} \sum_{s=0}^{t-3}{\rho}_{ab}^s\right)\cdot \left( 1+ C_a\sum_{s' = 0}^{s-2}\rho_a^{s'} \right) \cdot \epsilon  \\
    \leq & ~ L_g\Sigma_a^{(t)}\left( 1 + C_{bg}\bar{\Sigma}_{ab}^{(t-1)}\right) \cdot \epsilon.
	\end{align*}	
\end{proof}

\section{Experiment}
\label{sec:exp}

\begin{figure}[b]
	\begin{center}
		\includegraphics[scale=0.58]{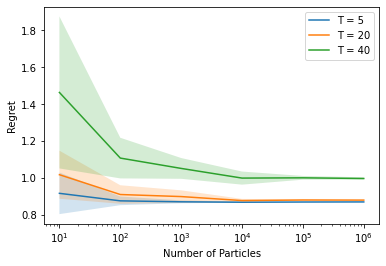}
	\end{center}
		\caption{Relationship between the regret and the number of particles}
	\label{fig}
\end{figure}

In this section, we use simulations to show the error of particle filtering can accumulate and be amplified through sequential planning. We run a process with a maximum time step $T = 40$ and $d = 1$. Since our bound shows the number of particles needed is insensitive to the dimension $d$, we mainly show how the number of time steps can affect the accuracy of particle filtering. 

We consider the following process, for all $t \in [T] $,  
\begin{align*}
 x_{t} =  x_{t-1}  + u_{t-1} + \xi_{t-1}, \; \text{ and }o_t = x_t + \zeta_t.
\end{align*}
The process can suffer a random shift of size 1, i.e., $\xi_t$ follows a uniform distribution on set $\{0, 1\}$.  $\zeta_t$ follows the standard normal distribution $\mathcal{N}(0, 1)$. The regret is defined as the average $\ell_1$ norm of the states, i.e., $r(x_{1:t}) = \sum_{i=1}^t |x_t|/t$. The policy function is $g(x)=-x$.
We show in Figure~\ref{fig} the regret and its standard deviation of the estimation using different number of particles. The result shows the number of particles needed for an accurate estimation can increase fast as the number of time step increases due to error accumulation. This experiment corroborates the importance of our theoretical results.

%% file: du_284.bbl
\begin{thebibliography}{35}
\providecommand{\natexlab}[1]{#1}
\providecommand{\url}[1]{\texttt{#1}}
\expandafter\ifx\csname urlstyle\endcsname\relax
  \providecommand{\doi}[1]{doi: #1}\else
  \providecommand{\doi}{doi: \begingroup \urlstyle{rm}\Url}\fi

\bibitem[Abbasi-Yadkori et~al.(2014)Abbasi-Yadkori, Bartlett, and
  Kanade]{abbasi2014tracking}
Yasin Abbasi-Yadkori, Peter Bartlett, and Varun Kanade.
\newblock Tracking adversarial targets.
\newblock In \emph{International Conference on Machine Learning}, pages
  369--377, 2014.

\bibitem[Agarwal et~al.(2019{\natexlab{a}})Agarwal, Bullins, Hazan, Kakade, and
  Singh]{agarwal2019online}
Naman Agarwal, Brian Bullins, Elad Hazan, Sham~M Kakade, and Karan Singh.
\newblock Online control with adversarial disturbances.
\newblock \emph{arXiv preprint arXiv:1902.08721}, 2019{\natexlab{a}}.

\bibitem[Agarwal et~al.(2019{\natexlab{b}})Agarwal, Hazan, and
  Singh]{agarwal2019logarithmic}
Naman Agarwal, Elad Hazan, and Karan Singh.
\newblock Logarithmic regret for online control.
\newblock In \emph{Advances in Neural Information Processing Systems}, pages
  10175--10184, 2019{\natexlab{b}}.

\bibitem[Bertsekas et~al.(1995)Bertsekas, Bertsekas, Bertsekas, and
  Bertsekas]{bertsekas1995dynamic}
Dimitri~P Bertsekas, Dimitri~P Bertsekas, Dimitri~P Bertsekas, and Dimitri~P
  Bertsekas.
\newblock \emph{Dynamic programming and optimal control}, volume~1.
\newblock Athena scientific Belmont, MA, 1995.

\bibitem[Bruder et~al.(2019)Bruder, Gillespie, Remy, and
  Vasudevan]{bruder2019modeling}
Daniel Bruder, Brent Gillespie, C~David Remy, and Ram Vasudevan.
\newblock Modeling and control of soft robots using the koopman operator and
  model predictive control.
\newblock In \emph{Robotics: Science and Systems (RSS)}, 2019.

\bibitem[Brunton et~al.(2016)Brunton, Brunton, Proctor, and
  Kutz]{brunton2016koopman}
Steven~L Brunton, Bingni~W Brunton, Joshua~L Proctor, and J~Nathan Kutz.
\newblock Koopman invariant subspaces and finite linear representations of
  nonlinear dynamical systems for control.
\newblock \emph{PloS one}, 11\penalty0 (2), 2016.

\bibitem[Chopin et~al.(2004)]{chopin2004central}
Nicolas Chopin et~al.
\newblock Central limit theorem for sequential monte carlo methods and its
  application to bayesian inference.
\newblock \emph{The Annals of Statistics}, 32\penalty0 (6):\penalty0
  2385--2411, 2004.

\bibitem[Cohen et~al.(2018)Cohen, Hassidim, Koren, Lazic, Mansour, and
  Talwar]{cohen2018online}
Alon Cohen, Avinatan Hassidim, Tomer Koren, Nevena Lazic, Yishay Mansour, and
  Kunal Talwar.
\newblock Online linear quadratic control.
\newblock \emph{arXiv preprint arXiv:1806.07104}, 2018.

\bibitem[Crisan and Doucet(2002)]{crisan2002survey}
Dan Crisan and Arnaud Doucet.
\newblock A survey of convergence results on particle filtering methods for
  practitioners.
\newblock \emph{IEEE Transactions on Signal Processing}, 50\penalty0
  (3):\penalty0 736--746, 2002.

\bibitem[Dean et~al.(2019)Dean, Matni, Recht, and Ye]{dean2019robust}
Sarah Dean, Nikolai Matni, Benjamin Recht, and Vickie Ye.
\newblock Robust guarantees for perception-based control.
\newblock \emph{arXiv preprint arXiv:1907.03680}, 2019.

\bibitem[Even-Dar et~al.(2009)Even-Dar, Kakade, and Mansour]{even2009online}
Eyal Even-Dar, Sham~M Kakade, and Yishay Mansour.
\newblock Online markov decision processes.
\newblock \emph{Mathematics of Operations Research}, 34\penalty0 (3):\penalty0
  726--736, 2009.

\bibitem[Foster and Simchowitz(2020)]{foster2020logarithmic}
Dylan~J Foster and Max Simchowitz.
\newblock Logarithmic regret for adversarial online control.
\newblock \emph{arXiv preprint arXiv:2003.00189}, 2020.

\bibitem[Goel and Hassibi(2020)]{goel2020power}
Gautam Goel and Babak Hassibi.
\newblock The power of linear controllers in lqr control.
\newblock \emph{arXiv preprint arXiv:2002.02574}, 2020.

\bibitem[Hausknecht and Stone(2015)]{hausknecht2015deep}
Matthew Hausknecht and Peter Stone.
\newblock Deep recurrent {Q}-learning for partially observable {MDPs}.
\newblock In \emph{2015 AAAI Fall Symposium Series}, 2015.

\bibitem[Hazan et~al.(2017)Hazan, Singh, and Zhang]{hazan2017learning}
Elad Hazan, Karan Singh, and Cyril Zhang.
\newblock Learning linear dynamical systems via spectral filtering.
\newblock In \emph{Advances in Neural Information Processing Systems}, pages
  6702--6712, 2017.

\bibitem[Huggins and Roy(2019)]{huggins2019sequential}
Jonathan~H Huggins and Daniel~M Roy.
\newblock Sequential monte carlo as approximate sampling: bounds, adaptive
  resampling via $\infty$-ess, and an application to particle gibbs.
\newblock \emph{Bernoulli}, 25\penalty0 (1):\penalty0 584--622, 2019.

\bibitem[Jonschkowski et~al.(2018)Jonschkowski, Rastogi, and
  Brock]{jonschkowski2018differentiable}
Rico Jonschkowski, Divyam Rastogi, and Oliver Brock.
\newblock Differentiable particle filters: End-to-end learning with algorithmic
  priors.
\newblock In \emph{Robotics: Science and Systems (RSS)}, 2018.

\bibitem[Kaelbling et~al.(1998)Kaelbling, Littman, and
  Cassandra]{kaelbling1998planning}
Leslie~Pack Kaelbling, Michael~L Littman, and Anthony~R Cassandra.
\newblock Planning and acting in partially observable stochastic domains.
\newblock \emph{Artificial intelligence}, 101\penalty0 (1-2):\penalty0 99--134,
  1998.

\bibitem[Kalman(1960)]{kalman1960new}
Rudolph~Emil Kalman.
\newblock A new approach to linear filtering and prediction problems.
\newblock \emph{Transactions of the ASME--Journal of Basic Engineering},
  82\penalty0 (Series D):\penalty0 35--45, 1960.

\bibitem[Karkus et~al.(2017)Karkus, Hsu, and Lee]{karkus2017qmdp}
Peter Karkus, David Hsu, and Wee~Sun Lee.
\newblock Qmdp-net: Deep learning for planning under partial observability.
\newblock In \emph{Advances in Neural Information Processing Systems
  (NeurIPS)}, 2017.

\bibitem[Karkus et~al.(2018{\natexlab{a}})Karkus, Hsu, and
  Lee]{karkus2018integrating}
Peter Karkus, David Hsu, and Wee~Sun Lee.
\newblock Integrating algorithmic planning and deep learning for partially
  observable navigation.
\newblock \emph{arXiv preprint arXiv:1807.06696}, 2018{\natexlab{a}}.

\bibitem[Karkus et~al.(2018{\natexlab{b}})Karkus, Hsu, and
  Lee]{karkus2018particle}
Peter Karkus, David Hsu, and Wee~Sun Lee.
\newblock Particle filter networks with application to visual localization.
\newblock In \emph{Conference on Robot Learning (CoRL)}, 2018{\natexlab{b}}.

\bibitem[Kitagawa(1993)]{kitagawa1993monte}
Genshiro Kitagawa.
\newblock Monte carlo filtering and smoothing method for non-gaussian nonlinear
  state space model.
\newblock \emph{Inst. Statist. Math. Res. Memo.}, 1993.

\bibitem[Li et~al.(2019)Li, Chen, and Li]{li2019online}
Yingying Li, Xin Chen, and Na~Li.
\newblock Online optimal control with linear dynamics and predictions:
  Algorithms and regret analysis.
\newblock In \emph{Advances in Neural Information Processing Systems}, pages
  14858--14870, 2019.

\bibitem[Marion and Schmidler(2018)]{marion2018finite}
Joseph Marion and Scott~C Schmidler.
\newblock Finite sample complexity of sequential monte carlo estimators.
\newblock \emph{arXiv preprint arXiv:1803.09365}, 2018.

\bibitem[Neu and G{\'o}mez(2017)]{neu2017fast}
Gergely Neu and Vicen{\c{c}} G{\'o}mez.
\newblock Fast rates for online learning in linearly solvable markov decision
  processes.
\newblock \emph{arXiv preprint arXiv:1702.06341}, 2017.

\bibitem[Oreshkin et~al.(2011)Oreshkin, Coates, et~al.]{oreshkin2011analysis}
Boris~N Oreshkin, Mark~J Coates, et~al.
\newblock Analysis of error propagation in particle filters with approximation.
\newblock \emph{The Annals of Applied Probability}, 21\penalty0 (6):\penalty0
  2343--2378, 2011.

\bibitem[Platt~Jr et~al.(2010)Platt~Jr, Tedrake, Kaelbling, and
  Lozano-Perez]{platt2010belief}
Robert Platt~Jr, Russ Tedrake, Leslie Kaelbling, and Tomas Lozano-Perez.
\newblock Belief space planning assuming maximum likelihood observations.
\newblock In \emph{Robotics: Science and Systems (RSS)}, 2010.

\bibitem[Simchowitz(2020)]{simchowitz2020making}
Max Simchowitz.
\newblock Making non-stochastic control (almost) as easy as stochastic.
\newblock \emph{arXiv preprint arXiv:2006.05910}, 2020.

\bibitem[Simchowitz et~al.(2020)Simchowitz, Singh, and
  Hazan]{simchowitz2020improper}
Max Simchowitz, Karan Singh, and Elad Hazan.
\newblock Improper learning for non-stochastic control.
\newblock \emph{arXiv preprint arXiv:2001.09254}, 2020.

\bibitem[Tropp(2015)]{t15}
Joel~A Tropp.
\newblock An introduction to matrix concentration inequalities.
\newblock \emph{Foundations and Trends{\textregistered} in Machine Learning},
  8\penalty0 (1-2):\penalty0 1--230, 2015.

\bibitem[Tsiamis et~al.(2020)Tsiamis, Matni, and Pappas]{tsiamis2020sample}
Anastasios Tsiamis, Nikolai Matni, and George Pappas.
\newblock Sample complexity of kalman filtering for unknown systems.
\newblock In \emph{Learning for Dynamics and Control}, pages 435--444. PMLR,
  2020.

\bibitem[Wang et~al.(2019)Wang, Liu, Wu, Zhu, Du, Fei-Fei, and
  Tenenbaum]{wang2019dual}
Yunbo Wang, Bo~Liu, Jiajun Wu, Yuke Zhu, Simon~S Du, Li~Fei-Fei, and Joshua~B
  Tenenbaum.
\newblock Dual sequential monte carlo: Tunneling filtering and planning in
  continuous {POMDPs}.
\newblock \emph{arXiv preprint arXiv:1909.13003}, 2019.

\bibitem[Whiteley et~al.(2016)Whiteley, Lee, Heine, et~al.]{whiteley2016role}
Nick Whiteley, Anthony Lee, Kari Heine, et~al.
\newblock On the role of interaction in sequential monte carlo algorithms.
\newblock \emph{Bernoulli}, 22\penalty0 (1):\penalty0 494--529, 2016.

\bibitem[Yu et~al.(2009)Yu, Mannor, and Shimkin]{yu2009markov}
Jia~Yuan Yu, Shie Mannor, and Nahum Shimkin.
\newblock Markov decision processes with arbitrary reward processes.
\newblock \emph{Mathematics of Operations Research}, 34\penalty0 (3):\penalty0
  737--757, 2009.

\end{thebibliography}
